\documentclass{article}

\PassOptionsToPackage{numbers, sort, compress}{natbib}

\usepackage[final]{neurips_2021}

\usepackage[utf8]{inputenc} 
\usepackage[T1]{fontenc}    
\usepackage{hyperref}       
\usepackage{url}            
\usepackage{booktabs}       
\usepackage[tbtags]{mathtools}
\usepackage{amssymb}
\usepackage{amsthm}
\usepackage{nicefrac}       
\usepackage{microtype}      
\usepackage{xcolor}         
\usepackage{multirow}
\usepackage{caption}
\usepackage{subcaption}
\usepackage{wrapfig}
\usepackage[ruled,vlined]{algorithm2e}
\hypersetup{colorlinks,citecolor=blue}

\usepackage{minitoc}

\mtcsetrules{*}{off}

\usepackage[utf8]{inputenc}   
\usepackage[T1]{fontenc}      

\usepackage{comment}
\usepackage{natbib}

\allowdisplaybreaks
\usepackage{amssymb,mathrsfs}
\usepackage{amsfonts}
\usepackage{upgreek}
\usepackage{xspace}
\usepackage{graphicx}
\usepackage{color}

\usepackage{stmaryrd}
\usepackage[inline]{enumitem}
\usepackage{url}

\usepackage{tikz}
\usetikzlibrary{calc}

\usepackage{pgfplots}
\usepackage{xcolor}
\usepackage{bbm}
\usepackage{ifthen}
\usepackage{xargs}
\usepackage[textwidth=1.8cm]{todonotes}

\usepackage{aliascnt}
\usepackage{cleveref}
\usepackage{autonum}
\makeatletter

\crefname{theorem}{theorem}{Theorems}
\Crefname{Theorem}{Theorem}{Theorems}

\newtheorem*{lemma_nonumber*}{Lemma}

\newaliascnt{lemma}{theorem}

\aliascntresetthe{lemma}
\crefname{lemma}{lemma}{lemmas}
\Crefname{Lemma}{Lemma}{Lemmas}

\newaliascnt{corollary}{theorem}

\aliascntresetthe{corollary}
\crefname{corollary}{corollary}{corollaries}
\Crefname{Corollary}{Corollary}{Corollaries}

\newaliascnt{proposition}{theorem}
\newtheorem{proposition}[proposition]{Proposition}
\aliascntresetthe{proposition}
\crefname{proposition}{proposition}{propositions}
\Crefname{Proposition}{Proposition}{Propositions}

\newaliascnt{definition}{theorem}

\aliascntresetthe{definition}
\crefname{definition}{definition}{definitions}
\Crefname{Definition}{Definition}{Definitions}

\newaliascnt{remark}{theorem}

\aliascntresetthe{remark}
\crefname{remark}{remark}{remarks}
\Crefname{Remark}{Remark}{Remarks}

\crefname{example}{example}{examples}
\Crefname{Example}{Example}{Examples}

\crefname{figure}{figure}{figures}
\Crefname{Figure}{Figure}{Figures}

\crefformat{assumption}{{\textbf{A}}#2#1#3}

\newtheorem{assumptionF}{\textbf{F}\hspace{-3pt}}
\crefformat{assumptionF}{{\textbf{F}}#2#1#3}

\Crefname{assumptionB}{\textbf{B}\hspace{-3pt}}{\textbf{B}\hspace{-3pt}}
\crefname{assumptionB}{\textbf{B}}{\textbf{B}}

\Crefname{assumptionC}{\textbf{C}\hspace{-3pt}}{\textbf{C}\hspace{-3pt}}
\crefname{assumptionC}{\textbf{C}}{\textbf{C}}

\Crefname{assumptionH}{\textbf{H}\hspace{-3pt}}{\textbf{H}\hspace{-3pt}}
\crefname{assumptionH}{\textbf{H}}{\textbf{H}}

\Crefname{assumptionT}{\textbf{T}\hspace{-3pt}}{\textbf{T}\hspace{-3pt}}
\crefname{assumptionT}{\textbf{T}}{\textbf{T}}

\Crefname{assumptionT}{\textbf{T}\hspace{-3pt}}{\textbf{T}\hspace{-3pt}}
\crefname{assumptionT}{\textbf{T}}{\textbf{T}}

\Crefname{assumptionL}{\textbf{L}\hspace{-3pt}}{\textbf{L}\hspace{-3pt}}
\crefname{assumptionL}{\textbf{L}}{\textbf{L}}

\Crefname{assumptionQ}{\textbf{Q}\hspace{-3pt}}{\textbf{Q}\hspace{-3pt}}
\crefname{assumptionQ}{\textbf{Q}}{\textbf{Q}}

\Crefname{assumptionAR}{\textbf{AR}\hspace{-3pt}}{\textbf{AR}\hspace{-3pt}}
\crefname{assumptionAR}{\textbf{AR}}{\textbf{AR}}

\usepackage{bm}
 
\DeclareMathOperator{\E}{\mathbb{E}}
\title{Online Variational Filtering and Parameter Learning}

\newcommand{\Affiliation}{%
\end{tabular}\\\begin{tabular}[t]{c}\ignorespaces%
}

\author{
Andrew Campbell \thanks{Equal contribution} \And Yuyang Shi \footnotemark[1] \And Tom Rainforth \And Arnaud Doucet \Affiliation
   Department of Statistics, University of Oxford, UK \Affiliation
   \texttt{\{campbell, yshi, rainforth, doucet\}@stats.ox.ac.uk}
}

\begin{document}

\doparttoc
\faketableofcontents

\maketitle

\begin{abstract}
 We present a variational method for \emph{online} state estimation and parameter learning in state-space models (SSMs), a ubiquitous class of latent variable models for sequential data. As per standard batch variational techniques, we use stochastic gradients to simultaneously optimize a lower bound on the log evidence with respect to both model parameters and a variational approximation of the states' posterior distribution.
 However, unlike existing approaches, our method is able to operate in an entirely online manner, such that historic observations do not require revisitation after being incorporated 
 and the cost of updates at each time step remains constant, despite the growing dimensionality of the joint posterior distribution of the states.
 This is achieved by utilizing backward decompositions of this joint posterior distribution and of its variational approximation, combined with Bellman-type recursions for the evidence lower bound and its gradients.
 We demonstrate the performance of this methodology across several examples, including high-dimensional SSMs and sequential Variational Auto-Encoders.
\end{abstract}

\section{Introduction}
Many tasks in machine learning with time series data---such as video prediction \citep{yingzhen2018disentangled,gregor2018temporal}, speech enhancement \citep{richter2020speech}  or robot localization \citep{corenflos2021differentiable,jonschkowski2018differentiable,ma2020discriminative}---often need to be performed online. Online techniques are also necessary in contexts as diverse as target tracking \citep{beard2020solution}, weather prediction \citep{evensen2009data} and financial forecasting \citep{tsay2005analysis}. A popular class of models for these sequential data are SSMs which, when combined with neural network ideas, can also be used to define powerful sequential Variational Auto-Encoders (VAEs); see e.g. \cite{chung2015recurrent,fraccaro2016sequential,gregor2018temporal,marino2018general}. However, performing inference in SSMs is a challenging problem and approximate inference techniques for such models remain an active research area. 
 
Formally, an SSM is described by a latent Markov process and an observation process. Even if the model parameters are assumed known, online inference of the states of the latent process is a complex problem known as filtering. Standard approximations such as the extended Kalman Filter (KF), ensemble KF, and unscented KF can be used, but only provide an ad hoc Gaussian approximation to the filter \citep{evensen2009data,ito2000gaussian,sarkka2013bayesian}. More generally, assumed density filtering techniques \citep{azimi2005approximate,brigo1998differential,minka2001expectation} can provide other simple parametric family approximations.
These approximate filtering methods can be used, in turn, to develop online parameter learning procedures by either augmenting the state with the static parameters or using gradient-based approaches.
However, such approaches are notoriously unreliable \cite{kantas2015particle}. 
Particle Filtering (PF) methods, on the other hand, provide a more principled approach for online state and parameter estimation with theoretical guarantees~\citep{delmoral2004,douc2014nonlinear,kantas2015particle,tadic2020asymptotic},
but the variance of PF estimates typically scales exponentially with the state dimension~\citep{bengtsson2008curse}. 
 
Although they typically do not return consistent estimates, variational techniques provide an attractive alternative for simultaneous state estimation and parameter learning which scales better to high-dimensional latent states than PF, and are not restricted to simple parametric approximations. Many such methods have been proposed for SSMs over recent years, e.g. \cite{archer2015black,krishnan2017structured,maddison2017filtering,le2017auto,naesseth2017variational,kim2020variational,courts2020variational,ryderneural2021}. However, they have generally been developed for \emph{batch inference} where one maximizes the Evidence Lower Bound (ELBO) for a \emph{fixed} dataset.
 As such, they are ill-suited for \emph{online} learning as, whenever a new observation is collected, one would need to update the entire joint variational states distribution whose dimension increases over time.
 Though a small number of online variational approaches have been developed~\cite{smidl2008variational,marino2018general,zhao2020variational}, these rely on significant restrictions of the variational family, leading to approximations that cannot faithfully approximate the posterior distribution of the latent states.

 The main contribution of this paper is a novel variational approach to perform \emph{online} filtering and parameter learning for SSMs which bypasses those restrictions. As per standard batch variational inference, we simultaneously maximize an ELBO with respect to both model parameters and a variational approximation of the joint state posterior. However, our method operates in an entirely online manner and the cost of updates at each time step remains constant. Key to our approach is a backward decomposition of the variational approximation of the states posterior, combined with a representation of the ELBO and its gradients as expectations of value functions satisfying Bellman-type recursions akin to those appearing in Reinforcement Learning (RL) \citep{sutton2018reinforcement}.

\section{Background}
\label{sec:SSMandVI}
\paragraph{State-Space Models.}
SSMs are defined by a latent $\mathcal{X}$-valued Markov process $(x_t)_{t\geq1}$ and $\mathcal{Y}$-valued observations $(y_t)_{t\geq1}$, which are conditionally independent given the latent process. They correspond to the generative model
\begin{equation}\label{eq:SSMmmodel}
    x_1\sim \mu_{\theta}(x_1),\quad\quad x_{t+1}|x_t\sim f_{\theta}(x_{t+1}|x_t),\quad\quad y_t|x_t \sim g_{\theta}(y_t|x_t),
\end{equation}
where $\theta \in \mathbb{R}^{d_\theta}$ is a parameter of interest and we consider here $\mathcal{X}=\mathbb{R}^{d_x}$. For $y^t:=y_{1:t}=(y_1,...,y_t)$, we thus have  
\begin{equation}\label{eq:jointxy}
p_\theta(x_{1:t},y^t)=\mu_\theta(x_1)g_{\theta}(y_1|x_1)\prod\nolimits_{k=2}^t f_\theta(x_k|x_{k-1})g_\theta(y_k|x_k).
\end{equation}
Assume for the time being that $\theta$ is known. Given observations $(y_t)_{t\geq1}$ and parameter values $\theta$, one can perform online state inference by computing the posterior of $x_t$ given $y^t$ which satisfies
\begin{align}\label{eq:Bayesfilter}
  p_{\theta}(x_t|y^{t})=\frac{g_{\theta}(y_t|x_t)p_{\theta}(x_t|y^{t-1})}{p_{\theta}(y_t|y^{t-1})},              ~~~p_{\theta}(x_{t}|y^{t-1})=\!\int\! f_{\theta}(x_t|x_{t-1}) p_{\theta}(x_{t-1}|y^{t-1})\textrm{d}x_{t-1},
\end{align}
with $p_{\theta}(x_1|y^0):=\mu_{\theta}(x_1)$.
The log evidence $\ell_t(\theta):=\log p_{\theta}(y^t)$ is then given by
\begin{equation}\label{eq:logevidence}
    \ell_t(\theta)=\textstyle{\sum_{k=1}^t} \log  p_{\theta}(y_k|y^{k-1}),\quad \text{where}\quad p_{\theta}(y_k|y^{k-1})=\int g_{\theta}(y_k|x_k) p_\theta(x_k|y^{k-1})\mathrm{d}x_{k}.
\end{equation}
Here $p_{\theta}(x_k|y^k)$ is known as the filtering distribution. While the recursion (\ref{eq:Bayesfilter}) and the sequential decomposition (\ref{eq:logevidence}) are at the core of most existing online state and parameter inference techniques \citep{sarkka2013bayesian,kantas2015particle,zhao2020variational}, we will focus here on the joint state posterior distribution, also known as the smoothing distribution, of the states $x_{1:t}$ given $y^t$ and the corresponding representation of the log evidence 
\begin{equation}
    p_{\theta}(x_{1:t}|y^t)=p_\theta(x_{1:t},y^t)/p_{\theta}(y^t),\quad\quad \ell_t(\theta)= \log  p_{\theta}(y^t)= \log \left(\textstyle{\int} p_\theta(x_{1:t},y^t) \mathrm{d}x_{1:t}\right).
\end{equation}
\paragraph{Variational Inference.}
For general SSMs, the filtering and smoothing distributions as well as the log evidence are not available analytically and need to be approximated. For data $y^t$, standard variational approaches use stochastic gradient techniques to maximize the following ELBO
\begin{align}\label{ELBOt}
    \mathcal{L}_{t}(\theta,\phi):=\mathbb{E}_{q^{\phi}(x_{1:t}|y^t)}\left[\log \left(p_{\theta}(x_{1:t},y^t)\big/ q^{\phi}(x_{1:t}|y^t)\right)\right]\leq \ell_t(\theta).
\end{align}
Maximizing this ELBO w.r.t. the parameter $\phi$ of the variational distribution $q^{\phi}(x_{1:t}|y^t)$  corresponds to doing variational smoothing while maximizing it w.r.t. $\theta$ corresponds to doing parameter learning.

As the true smoothing distribution satisfies $p_\theta(x_{1:t}|y^t)=p_\theta(x_1|y^t)\prod_{k=2}^t p_\theta(x_{k}|y^t,x_{k-1})$, one typically uses  $q^{\phi}(x_{1:t}|y^t)=q^{\phi}(x_1|y^t)\prod_{k=2}^tq^{\phi}(x_{k}|y^t,x_{k-1})$ for the variational smoothing distribution; see e.g.  \citep{archer2015black,krishnan2017structured,kim2020variational,weber2015reinforced}. However this approach is not suitable for \emph{online} variational filtering and parameter learning. Firstly, the resulting marginal $q^{\phi}(x_{t}|y^t)$ of $q^{\phi}(x_{1:t}|y^t)$, approximating $p_\theta(x_t|y^t)$, is typically not available analytically. Secondly, when the new observation $y_{t+1}$ is collected, this approach would require recomputing an entirely new variational smoothing distribution with a dimension that increases with time. 
One can attempt to bypass this problem by restricting ourselves to $q^{\phi}(x_{1:t}|y^t)=q^{\phi}(x_1|y^1)\prod_{k=2}^t q^{\phi}(x_{k}|y^k,x_{k-1})$ as per~\cite{marino2018general}.
However, the switch from conditioning on $y^t$ to $y^k$ prohibits learning an accurate approximation of $p_\theta(x_{1:t}|y^t)$ as this formulation does not condition on all relevant data. 

\section{Online Variational Filtering and Parameter Learning}
\label{sec:Online Variational Inference using Reinforcement Learning}
Our online variational filtering and parameter estimation approach exploits a backward factorization of the variational smoothing distribution $q^\phi(x_{1:t}|y^t)$. The ELBO and its gradients w.r.t. $\theta$ and $\phi$ are computed in an \emph{online} manner as $t$ increases by using a combination of ideas from dynamic programming and RL, with a computational time that remains constant at each time step. 
To simplify notation, henceforth  we will write $q^{\phi}_t(x_{1:t})=q^\phi(x_{1:t}|y^t)$.

\subsection{Backward Decomposition of the Variational Smoothing Distribution}
\label{subsec: Backward Decomposition of the Variational Smoothing Distribution}
The key property that we will be exploiting is that $p_{\theta}(x_{1:t}|y^t)$ satisfies 
\begin{equation}\label{eq:backwarddecompojoint}
    p_{\theta}(x_{1:t}|y^t)=p_\theta(x_t|y^t) \prod_{k=1}^{t-1} p_\theta(x_k|y^k,x_{k+1}),\quad p_\theta(x_k|y^k,x_{k+1})=\frac{f_\theta(x_{k+1}|x_k)p_{\theta}(x_k|y^k)}{p_{\theta}(x_{k+1}|y^k)}.
\end{equation}
Equation (\ref{eq:backwarddecompojoint}) shows that, conditional upon $y^t$, $(x_k)_{k=1}^t$ is a reverse-time Markov chain of initial distribution $p_\theta(x_t|y^t)$ and backward Markov transition kernels $p_\theta(x_k|y^k,x_{k+1})$; see e.g. \cite{douc2014nonlinear,kantas2015particle}. Crucially the backward transition kernel at time $k$ depends only on the observations until time $k$.

To exploit this, we consider a variational smoothing distribution of the form
\begin{equation}\label{eq:variationaldistribution}
    q^\phi_t(x_{1:t}) = q^\phi_t(x_t) \prod\nolimits_{k=1}^{t-1} q^{\phi}_{k+1}(x_{k}|x_{k+1}),
\end{equation}
where $q^\phi_t(x_t)$ and $q^{\phi}_{k+1}(x_{k}|x_{k+1})$
are variational approximations of the filtering distribution $p_\theta(x_t|y^t)$ and the backward kernel $p_\theta(x_k|y^{k},x_{k+1})$ respectively. 
Using~\eqref{eq:variationaldistribution}, one now has 
\begin{align}
    \mathcal{L}_{t}(\theta,\phi)=&\,\ell_t(\theta) -\textup{KL}(q^\phi_t(x_t)||p_\theta(x_t|y^t))\\
    &-\sum\nolimits_{k=1}^{t-1} \E_{q^\phi_t(x_{k+1})} \left[  \textup{KL}(q^{\phi}_{k+1}(x_{k}|x_{k+1})||p_{\theta}(x_{k}|y^k,x_{k+1})) \right],
\end{align}
where $\textup{KL}$ is the Kullback--Leibler divergence and $q^{\phi}_t(x_{k+1})$ is the marginal distribution of $x_{k+1}$ under  the variational distribution $q^\phi_t(x_{1:t})$ defined in (\ref{eq:variationaldistribution}). Considering this variational distribution thus makes it possible to learn an arbitrarily accurate variational approximation of the true smoothing distribution whilst still only needing to condition on $y^k$ at time $k$ and not on future observations. Additionally, it follows directly from (\ref{eq:variationaldistribution}) that we can easily update $q^{\phi}_{t+1}(x_{1:t+1})$ from $q^{\phi}_{t}(x_{1:t})$ using
\begin{equation}\label{eq:recursionq_t}
   q^\phi_{t+1}(x_{1:t+1})=q^\phi_{t}(x_{1:t})m_{t+1}^{\phi}(x_{t+1}|x_t), \text{~~for~~} m_{t+1}^{\phi}(x_{t+1}|x_t):=\frac{q_{t+1}^{\phi}(x_{t}|x_{t+1})q_{t+1}^{\phi}(x_{t+1})}{q_{t}^{\phi}(x_t)}. 
\end{equation}
Here $m_{t+1}^{\phi}(x_{t+1}|x_t)$ can be viewed as an approximation of the Markov transition density $q^{\phi}_{t+1}(x_{t+1}|x_{t}) \propto q^{\phi}_{t+1}(x_{t}|x_{t+1})q^{\phi}_{t+1}(x_{t+1})$ but it is typically not a proper Markov transition density; i.e. $\int m_{t+1}^{\phi}(x_{t+1}|x_t)\textup{d} x_{t+1}\neq 1$ as $\int q^{\phi}_{t+1}(x_{t}|x_{t+1})q^{\phi}_{t+1}(x_{t+1}) \textup{d} x_{t+1}\neq q^{\phi}_t(x_t)$.

Let us assume that $q^{\phi}_k(x_{k})=q^{\phi_k}_k(x_{k})$ and $q^{\phi}_{k}(x_{k-1}|x_{k})=q^{\phi_{k}}_{k}(x_{k-1}|x_{k})$, then (\ref{eq:variationaldistribution}) and (\ref{eq:recursionq_t}) suggest that we only need to estimate $\phi_t$ at time $t$ as $y_t$ does not impact the backward Markov kernels prior to time $t$. However, we also have to be able to compute estimates of $\nabla_\phi \mathcal{L}_t(\theta, \phi)$ and $\nabla_\theta \mathcal{L}_t(\theta, \phi)$ to optimize parameters in a constant computational time at each time step, without having to consider the entire history of observations $y^t$. This is detailed in the next subsections where we show that the sequence of ELBOs $\{\mathcal{L}_{t}(\theta,\phi)\}_{t \geq 1}$ and its gradients  $\{\nabla_\theta \mathcal{L}_{t}(\theta,\phi)\}_{t \geq 1}$ and $\{\nabla_\phi \mathcal{L}_{t}(\theta,\phi)\}_{t \geq 1}$ can all be computed online when using the variational distributions $\{q^{\phi}_t(x_{1:t})\}_{t \geq 1}$ defined in (\ref{eq:variationaldistribution}).

\subsection{Forward recursion for the ELBO} \label{sec:ForwardRecursionForELBO}
We start by presenting a forward-only recursion for the computation of $\{\mathcal{L}_{t}(\theta,\phi)\}_{t \geq 1}$. This recursion illustrates the parallels between variational inference and RL and is introduced to build intuition.

\begin{proposition}\label{prop:ELBOrecursion}
The ELBO $\mathcal{L}_{t}(\theta,\phi)$ satisfies for $t\geq 1$
\begin{equation}\label{eq:ELBOexpectationV}
\mathcal{L}_{t}(\theta,\phi)=\mathbb{E}_{q_{t}^{\phi}(x_t)}[V^{\theta,\phi}_{t}(x_t)]
\quad\text{for}\quad V^{\theta,\phi}_{t}(x_t):=\mathbb{E}_{q_{t}^{\phi}(x_{1:t-1}|x_{t})}\left[ \log\left(p_{\theta}(x_{1:t},y^{t})\big/q_{t}^{\phi}(x_{1:t})\right)\right],
\end{equation}
with the convention $V^{\theta,\phi}_{1}(x_1):=r^{\theta,\phi}_1(x_0,x_1):=\log(p_\theta(x_1,y_1)/q^{\phi}_1(x_1))$. Additionally, we have
\begin{align}
\label{eq:recursionV}
V^{\theta,\phi}_{t+1}(x_{t+1})=&\,\mathbb{E}_{q_{t+1}^{\phi}(x_{t}|x_{t+1})}[V^{\theta,\phi}_{t}(x_{t})+r_{t+1}^{\theta,\phi}(x_{t},x_{t+1})], \quad \text{where} \\
\label{eq:reward}
    r_{t+1}^{\theta,\phi}(x_{t},x_{t+1}):=&\,\log \left(f_\theta(x_{t+1}|x_t)g_\theta(y_{t+1}|x_{t+1}) \big/ m_{t+1}^{\phi}(x_{t+1}|x_t)\right).
\end{align}
\vspace{-14pt}
\end{proposition}
All proofs are given in Appendix \ref{sec:Proofs}. Proposition \ref{prop:ELBOrecursion} shows that we can compute $\mathcal{L}_{t}(\theta,\phi)$, for $t\ge 1$, online by recursively computing the functions $V^{\theta,\phi}_t$ using (\ref{eq:recursionV}) and then taking the expectation of $V^{\theta,\phi}_t$ w.r.t. $q^{\phi}_t(x_t)$ to obtain the ELBO at time $t$. 
Thus, given $V^{\theta,\phi}_{t}$, we can compute $V^{\theta,\phi}_{t+1}$ and $\mathcal{L}_{t+1}(\theta,\phi)$ using only $y_{t+1}$, with a cost that remains constant in $t$. 

This type of recursion is somewhat similar to those appearing in RL. We can indeed think of the ELBO $\mathcal{L}_{t}(\theta,\phi)$ as the expectation of a sum of ``rewards'' $r^{\theta,\phi}_k$ given in (\ref{eq:reward}) from $k=1$ to $k=t$ which we compute recursively using the ``value'' function $V^{\theta,\phi}_t$. However, while in RL the value function is given by the expectation of the sum of future rewards starting from $x_t$, the value function defined here is the expectation of the sum of past rewards conditional upon arriving in $x_t$. This yields the required forward recursion instead of a backward recursion. We expand on this link in Section \ref{sec:relatedwork} and Appendix \ref{sec:LinkWtihRL}.

\subsection{Forward recursion for ELBO gradient w.r.t. $\theta$}
A similar recursion can be obtained to compute  $\{\nabla_\theta \mathcal{L}_{t}(\theta,\phi)\}_{t \geq 1}$. This  recursion will be at the core of our online parameter learning algorithm. Henceforth, we will assume that regularity conditions allowing both differentiation and the interchange of integration and differentiation are satisfied.
\begin{proposition}
\label{prop:ELBOgradtheta} 
The ELBO gradient $\nabla_{\theta} \mathcal{L}_{t}(\theta,\phi)$ satisfies for $t\geq 1$
\begin{equation}
\nabla_{\theta}\mathcal{L}_{t}(\theta,\phi)=\mathbb{E}_{q^{\phi}_t(x_t)}[S^{\theta,\phi}_t(x_t)],\quad\text{where}\quad S^{\theta,\phi}_t(x_t):=\nabla_\theta V^{\theta,\phi}_t(x_t).
\end{equation}
Additionally, if we define $ s^{\theta}_{t+1}(x_{t},x_{t+1}):=\nabla_\theta r_{t+1}^{\theta,\phi}(x_{t},x_{t+1}) = \nabla_\theta \log f_{\theta}(x_{t+1}|x_{t})g_\theta(y_{t+1}|x_{t+1})$
then 
\begin{align}\label{eq:Srecursion}
     S^{\theta,\phi}_{t+1}(x_{t+1})=\mathbb{E}_{q^{\phi}_{t+1}(x_{t}|x_{t+1})}\left[S^{\theta,\phi}_t(x_t)+s^{\theta}_{t+1}(x_t,x_{t+1})\right].
\end{align} 
\vspace{-14pt}
\end{proposition}
Proposition \ref{prop:ELBOgradtheta} shows that we can compute $\{\nabla_\theta \mathcal{L}_{t}(\theta,\phi)\}_{t \geq 1}$ online by propagating $\{S^{\theta,\phi}_t\}_{t \geq 1}$ using (\ref{eq:Srecursion}) and taking the expectation of the vector $S^{\theta,\phi}_t$ w.r.t. $q^{\phi}_t(x_t)$ to obtain the gradient at time $t$.
Similar ideas have been previously exploited in the statistics literature to obtain a forward recursion for the score vector $\nabla_{\theta} \ell_t(\theta)$ so as to perform recursive maximum likelihood parameter estimation; see e.g. \cite[Section 4]{kantas2015particle}. In this case, one has $\nabla_{\theta}\ell_{t}(\theta)=\mathbb{E}_{p_{\theta}(x_t|y^t)}[S^{\theta}_t(x_t)]$ where $S^{\theta}_t$ satisfies a recursion similar to (\ref{eq:Srecursion}) with $q^{\phi}_{t+1}(x_{t}|x_{t+1})$ replaced by $p_\theta(x_{t}|y^{t},x_{t+1})$.

\subsection{Forward recursion for ELBO gradient w.r.t. $\phi$}
We finally establish forward-only recursions for the gradient of the ELBO w.r.t. $\phi$ which will allow us to perform online variational filtering. 
We consider here the case where, for all $k$, $q^{\phi}_k(x_k)=q^{\phi_k}_k(x_k)$ and $q^{\phi}_k(x_{k-1}|x_k)=q^{\phi_k}_k(x_{k-1}|x_k)$ so $\mathcal{L}_t(\theta,\phi)=\mathcal{L}_t(\theta,\phi_{1:t})$ and  $V^{\theta,\phi}_{t}(x_t)=V^{\theta,\phi_{1:t}}_{t}(x_t)$. At time step $t$, we will optimize w.r.t. $\phi_t$ and hold all previous $\phi_{1:t-1}$ constant. Our overall variational posterior \eqref{eq:variationaldistribution} is denoted $ q_t^{\phi_{1:t}}(x_{1:t})$. The alternative approach where $\phi$ is shared through time (amortization) is investigated in Appendix \ref{sec:AppendixAmortization}.

Since the expectation is taken w.r.t. $q_t^{\phi_{1:t}}(x_{1:t})$ in $\mathcal{L}_{t}$, optimizing $\phi_t$ is slightly more difficult than for $\theta$. However, we can still derive a forward recursion for the $\phi$ gradients and we will leverage the reparameterization trick to reduce the variance of the gradient estimates; i.e. we assume that $x_t(\phi_t;\epsilon_t)\sim q^{\phi_t}_t(x_t)$  and $x_{t-1}(\phi_t;\epsilon_{t-1},x_t)\sim q^{\phi_t}_t(x_{t-1}|x_t)$ when $\epsilon_{t-1}\sim \lambda(\epsilon),~\epsilon_t \sim \lambda(\epsilon)$. 

\begin{proposition}\label{prop:ELBOgradphi}
The ELBO gradient $\nabla_{\phi_t} \mathcal{L}_{t}(\theta,\phi_{1:t})$ satisfies for $t\geq 1$
\begin{align}
\nabla_{\phi_t} \mathcal{L}_{t}(\theta,\phi_{1:t})=\nabla_{\phi_t} \mathbb{E}_{q_{t}^{\phi_t}(x_t)}[V^{\theta,\phi_{1:t}}_{t}(x_t)]=\mathbb{E}_{\lambda(\epsilon_t)}[\nabla_{\phi_t} V^{\theta,\phi_{1:t}}_{t}(x_t(\phi_t; \epsilon_t))].
\end{align}

Additionally, one has
\begin{multline}\label{eq:recursionQphi}
\nabla_{\phi_{t+1}} V^{\theta,\phi_{1:t+1}}_{t+1}(x_{t+1}(\phi_{t+1}; \epsilon_{t+1})) \\
 = \mathbb{E}_{\lambda(\epsilon_t)} \left[ T^{\theta,\phi_{1:t}}_t(x_{t}(\phi_{t+1}; \epsilon_{t}, x_{t+1}(\phi_{t+1}; \epsilon_{t+1}))) \frac{\textup{d} x_{t}(\phi_{t+1}; \epsilon_{t}, x_{t+1}(\phi_{t+1}; \epsilon_{t+1}))}{\textup{d} \phi_{t+1}} \right. \\ 
+ \left. \nabla_{\phi_{t+1}} r_{t+1}^{\theta,\phi_{t:t+1}}(x_{t}(\phi_{t+1}; \epsilon_{t}, x_{t+1}(\phi_{t+1}; \epsilon_{t+1})),x_{t+1}(\phi_{t+1}; \epsilon_{t+1}))\right] , 
\end{multline}
where 
$ T^{\theta,\phi_{1:t}}_t(x_{t}) := \frac{\partial}{\partial x_t}V^{\theta,\phi_{1:t}}_{t}(x_t)$ satisfies the forward recursion  
\begin{multline}\label{eq:recursionT}
T^{\theta,\phi_{1:t+1}}_{t+1}(x_{t+1}) =  \mathbb{E}_{\lambda(\epsilon_t)}\left[T^{\theta,\phi_{1:t}}_t(x_{t}(\phi_{t+1}; \epsilon_{t}, x_{t+1})) \frac{\partial x_{t}(\phi_{t+1}; \epsilon_{t}, x_{t+1})}{\partial x_{t+1}}\right. \\ 
+ \left. \nabla_{x_{t+1}} r_{t+1}^{\theta,\phi_{t:t+1}}(x_{t}(\phi_{t+1}; \epsilon_{t}, x_{t+1}),x_{t+1})\right]. 
\end{multline}
Here, $\frac{\textup{d} x_{t}(\phi_{t+1}; \epsilon_{t}, x_{t+1}(\phi_{t+1}; \epsilon_{t+1}))}{\textup{d} \phi_{t+1}}, \frac{\partial x_{t}(\phi_{t+1}; \epsilon_{t}, x_{t+1})}{\partial x_{t+1}}$ are Jacobians of appropriate dimensions.

\end{proposition}

\subsection{Estimating the ELBO and its Gradients}
\label{sec:theta_gradients}
As we consider the case $q^{\phi}_k(x_k)=q^{\phi_k}_k(x_k),~q^{\phi}_k(x_{k-1}|x_k)=q^{\phi_k}_k(x_{k-1}|x_k)$, we have $S_t^{\theta, \phi}(x_t) = S_t^{\theta, \phi_{1:t}}(x_t)$. At time $t$ we optimize $\phi_t$ and hold $\phi_{1:t-1}$ constant. Practically, we are not able to compute in closed-form the functions $V^{\theta,\phi_{1:t}}_{t}(x_t)$, $S_t^{\theta, \phi_{1:t}}(x_t)$ and $T_t^{\theta, \phi_{1:t}}(x_t)$ appearing in the forward recursions of $ \mathcal{L}_t(\theta,\phi_{1:t})$,  $\nabla_\theta \mathcal{L}_t(\theta,\phi_{1:t})$ and $\nabla_\phi \mathcal{L}_t(\theta,\phi_{1:t})$ respectively. However, we can exploit the above recursions to approximate these functions online using 
regression as is commonly done in RL. We then show how to use these gradients for online filtering and parameter learning.

We approximate $S_{t+1}^{\theta, \phi_{1:t+1}}$ with $\hat{S}_{t+1}$. Equation (\ref{eq:Srecursion}) shows that $\hat{S}_{t+1}$ can be learned using $\hat{S}_{t}$ through regression of the simulated dataset\footnote{We define $\hat{S}_0 := 0$, $x_0^i := \emptyset$, $s_1^\theta(x_0, x_1) = \nabla_{\theta} \log \mu_\theta(x_1) g_\theta(y_1|x_1)$ and $q_1^{\phi_1}(x_0, x_1) := q_1^{\phi_1}(x_1)$.} $\{x_{t+1}^i, \hat{S}_{t}(x_{t}^i)+s_{t+1}^\theta(x_{t}^i, x_{t+1}^i)\}$
with $(x^i_t,x^i_{t+1})\overset{\textup{i.i.d.}}{\sim} q^{\phi_{t+1}}_{t+1}(x_t,x_{t+1})$ for $i=1,...,N$ (see Appendix \ref{sec:ObjectiveForRecursionFitting} for derivation). We can use neural networks to model $\hat{S}_t$ or Kernel Ridge Regression (KRR). Note that the use of KRR to estimate gradients for variational learning has recently been demonstrated by \cite{li2020amortised}.

We similarly approximate $T_{t+1}^{\theta, \phi_{1:t+1}}$ with $\hat{T}_{t+1}$.
As before, we can model $\hat{T}_{t+1}$ using neural networks or KRR. We use recursion (\ref{eq:recursionT}) and $\hat{T}_t$ to create the following dataset for regression\footnote{Similarly, we define $\hat{T}_0 := 0$, $x_0^i := \emptyset$ and $r_1^{\theta, \phi_{0:1}}$ as in Proposition \ref{prop:ELBOrecursion}.}
\begin{align}
    \Big\{x_{t+1}^i,~~  \hat{T}_{t}(x_{t}(\phi_{t+1}; \epsilon^i_{t},x_{t+1}^i)) \tfrac{\partial x_{t}(\phi_{t+1}; \epsilon^i_{t}, x_{t+1}^i)}{\partial x_{t+1}^i} + \nabla_{x_{t+1}^i} r_{t+1}^{\theta, \phi_{t:t+1}}(x_{t}(\phi_{t+1}; \epsilon^i_{t}, x_{t+1}^i), x_{t+1}^i) \Big\},
\end{align}
where $x^i_{t+1}\sim q^{\phi_{t+1}}_{t+1}(x_{t+1})$ and $\epsilon^i_{t}\sim \lambda(\epsilon)$ for $i=1,...,N$. The choice of the distribution over the inputs, $x_{t+1}^i$, for each simulated dataset is arbitrary but it will determine where the approximations are most accurate. We choose $q_{t+1}^{\phi_{t+1}}(x_{t+1})$ to best match where we expect the approximations to be evaluated, more details are given in Appendix \ref{sec:ObjectiveForRecursionFitting}.

Note that if one is interested in computing online an approximation of the ELBO, we can again similarly approximate $V^{\theta,\phi_{1:t}}_t(x_t)$ using regression to obtain $\hat{V}_t(x_t)$ by leveraging (\ref{eq:recursionV}). We will call the resulting approximate ELBO the Recursive ELBO (RELBO). We could also then differentiate $\hat{V}_t(x_t)$ w.r.t. $x_t$ to obtain an alternative method for estimating $T_t^{\theta, \phi_{1:t}}$ and optimizing $\phi_t$. However, as we are ultimately interested in accurate gradients, this approach does not exploit the readily available gradient information during the regression stage. 

By approximating $T_{t+1}^{\theta, \phi_{1:t+1}}$ with $\hat{T}_{t+1}$ and $S_{t+1}^{\theta,  \phi_{1:t+1}}$ with $\hat{S}_{t+1}$, we introduce some bias into our gradient estimates. We can trade bias for variance by using modified recursions; e.g. 
\begin{equation}
S^{\theta,\phi_{1:t+1}}_{t+1}(x_{t+1})=\textstyle{\mathbb{E}_{q^{\phi_{t-L+2:t+1}}_{t}(x_{t-L+1:t}|x_{t+1})} \left[S^{\theta,\phi_{1:t-L+1}}_{t-L+1}(x_{t-L+1})+\sum_{k=t-L+1}^t s^{\theta}_{k+1}(x_k,x_{k+1})\right]}.
\end{equation}
As $L$ increases, we will reduce bias but increase variance. Such ideas are also commonly used in RL but we will limit ourselves here to using $L=1$.

\subsection{Online Parameter Estimation}
Assume, for the sake of argument, that one has access to the log evidence $\ell_t(\theta)$ and that the observations arise from the SSM with parameter $\theta^{\star}$. Under regularity conditions, the average log-likelihood $\ell_t(\theta)/t$ converges as $t \rightarrow \infty$ towards a function $\ell(\theta)$  which is maximized at $\theta^{\star}$; see e.g. \cite{douc2014nonlinear,tadic2020asymptotic}. 
We can maximize this criterion online using Recursive Maximum Likelihood Estimation (RMLE) \citep{legland1997recursive,tadic2010analyticity,kantas2015particle,tadic2020asymptotic} which consists of updating the parameter estimate $\theta$ using
\begin{equation}\label{eq:RMLE}
\theta_{t+1}=\theta_{t}+\eta_{t+1} \left(\mathbb{E}_{p_{\theta_{0:t}}(x_t,x_{t+1}|y^{t+1})}[S_t(x_t)+s^{\theta_t}_{t+1}(x_t,x_{t+1})]-\mathbb{E}_{p_{\theta_{0:t-1}}(x_{t}|y^{t})}[S_{t}(x_{t})]\right).
\end{equation}
The difference of two expectations on the r.h.s. of \eqref{eq:RMLE} is an approximation of the gradient of the log-predictive $\log p_{\theta}(y_{t+1}|y^{t})$ evaluated at $\theta_t$. The approximation is given by $\nabla \log p_{\theta_{0:t}}(y^{t+1})-\nabla \log p_{\theta_{0:t-1}}(y^{t})$ with the notation $\nabla \log p_{\theta_{0:t}}(y^{t+1})$ corresponding to the expectation of the sum of terms $s^{\theta_{k}}_{k+1}(x_k,x_{k+1})$ w.r.t. the joint posterior states distribution defined by using the SSM with parameter $\theta_k$ at time $k+1$.

We proceed similarly in the variational context and update the parameter using 
\begin{align}
    \theta_{t+1} \!=\! \theta_t \!+\! \eta_{t+1} \Big(\! \E_{q_{t+1}^{\phi_{t+1}}(x_t, x_{t+1})} \left[ \hat{S}_t(x_t) \!+\! s_{t+1}^{\theta_t}(x_t, x_{t+1}) \right] \!-\! \E_{q_{t}^{\phi_{t}}(x_t)} \left[  \hat{S}_t(x_t)\right]\!\Big).
\end{align}
Here  $\hat{S}_t(x_t)$ approximates $S_t(x_t)$ satisfying $S_{t+1}(x_{t+1}):=\mathbb{E}_{q^{\phi_{t+1}}_{t+1}(x_t|x_{t+1})}[S_t(x_t)+s^{\theta_t}_{t+1}(x_t,x_{t+1})]$. We compute $\hat{S}_t$ as in Section \ref{sec:theta_gradients} with a simulated dataset using $\hat{S}_{t-1}$ and $\theta_{t-1}$.

\begin{algorithm}[t]
\SetAlgoNoLine
\DontPrintSemicolon
{\small \For{$t = 1, \dots, T$} {
    Initialize $\phi_t$ e.g. $\phi_t \leftarrow \phi_{t-1}$\;
    {\color{gray} \tcc{Update $\phi_t$ using $M$ stochastic gradient steps}}
    \For {$m=1, \dots, M$} { Sample $x_{t-1}^i, x_t^i \sim q^{\phi_t}_t(x_{t-1}, x_t)$ using reparameterization trick for $i=1,...,N$\;
        $\phi_t \leftarrow \phi_t + \gamma_m \frac{1}{N} \sum_{i=1}^N \{\hat{T}_{t-1}(x_{t-1}^i) \frac{\textup{d} x_{t-1}^i}{\textup{d} \phi_{t}} + \nabla_{\phi_t} r_t(x_{t-1}^i, x_t^i) \}$\;
    }
{\small \color{gray} \tcc{Update $\hat{T}_t(x_t)$ and $\hat{S}_t(x_t)$ as in Section \ref{sec:theta_gradients}}}
$\hat{T}_t(x_t) \overset{regression}{\leftarrow} \hat{T}_{t-1}(x_{t-1}) \frac{\partial x_{t-1}}{\partial x_t} + \nabla_{x_{t}} r_{t}(x_{t-1}, x_t)$.\;
$\hat{S}_t(x_t) \overset{regression}{\leftarrow} \hat{S}_{t-1}(x_{t-1}) + s_t^{\theta_{t-1}}(x_{t-1}, x_t)$. \;
{\small \color{gray} \tcc{Update $\theta$ using a stochastic gradient step}}
Sample $x_{t-1}^i, x_t^i \sim q^{\phi_t}_t(x_{t-1}, x_t), \quad \tilde{x}_{t-1}^i \sim q^{\phi_{t-1}}_{t-1}(x_{t-1})$ for $i=1,...,N$\;
$\theta_t \leftarrow \theta_{t-1} + \eta_t \frac{1}{N} \sum_{i=1}^N \{ \hat{S}_{t-1}(x_{t-1}^i) + s_t^{\theta_{t-1}}(x_{t-1}^i, x_t^i)  - \hat{S}_{t-1}(\tilde{x}_{t-1}^i) \}$\;
}
}
 \caption{Online Variational Filtering and Parameter Learning.}
\label{alg:TheAlgorithm}
\end{algorithm}

Putting everything together, we summarize our method in Algorithm \ref{alg:TheAlgorithm} using a simplified notation to help build intuition, a full description is given in Appendix \ref{sec:AppendixAlgorithm}. It is initialized using initial parameters $\phi_1, \theta_0$. We re-iterate the algorithm's computational cost does not grow with $t$. We need only store fixed size $\hat{T}$ and $\hat{S}$ models as well as the most recent $\phi_t$ and $\theta_t$ parameters. When performing backpropagation, $\hat{T}$ and $\hat{S}$ summarize all previous gradients, meaning we do not have to roll all the way back to $t=1$. Therefore, we only incur standard backpropagation computational cost w.r.t. $\phi_t$ and $\theta$. To scale to large $d_x$, we can use mean field $q_t^{\phi_t}(x_t)$, $q_t^{\phi_t}(x_{t-1}|x_t)$ keeping costs linear in $d_x$.
\vspace{-0.05cm}
\section{Related Work}
\label{sec:relatedwork}

As mentioned briefly in Section \ref{sec:ForwardRecursionForELBO}, our recursive method has close links with RL, which we make explicit here.
In `standard' RL, we have a value function and corresponding Bellman recursion given by
\begin{equation}\label{eq:backwardrecursionRL}
    V^{\text{RL}}_t (s_t) := \mathbb{E} \left[ \sum_{k=t}^T r(s_k, a_k) \right], \quad \quad V^{\text{RL}}_t(s_t) = \mathbb{E}_{s_{t+1}, a_t} \left[ r(s_t, a_t) + V^{\text{RL}}_{t+1}(s_{t+1}) \right],
\end{equation}
where $(s_t$, $a_t)$ is the state-action pair at time $t$. Whereas the RL value function summarizes future rewards and so recurses backward in time, our `value' function summarizes previous rewards and recurses forward in time. Writing this using the state-action notation, one obtains
\begin{equation}\label{eq:forwardrecursionRL}
    V_{t}(s_{t}) := \mathbb{E} \left[ \sum_{k=1}^{t} r(s_k, a_k) \right], \quad \quad V_{t+1}(s_{t+1}) = \mathbb{E}_{a_{t+1}, s_t} \left[ r(s_{t+1}, a_{t+1}) + V_t(s_t) \right].
\end{equation}
Note we have defined an \textit{anti-causal} graphical model, with $s_t$ depending on $s_{t+1}$ and $a_{t+1}$. If we further let $s_t = x_t$, $a_t = x_{t-1} \sim q^{\phi}_{t}(x_{t-1}|x_{t})$, $P(s_t|s_{t+1}, a_{t+1}) = \delta (s_t = a_{t+1})$ and $r(s_t, a_t) = r(x_t, x_{t-1})$ as in equation (\ref{eq:reward}), then the recursion in equation (\ref{eq:forwardrecursionRL}) for $V_t$ corresponds to the recursion in equation (\ref{eq:recursionV}) and $\mathcal{L}_{T} = \mathbb{E}_{q_T^{\phi}(x_T)}\left[ V_T(x_T) \right]$. We then differentiate this recursion with respect to $\theta$ and $x_{t+1}$ to get our gradient recursions (\ref{eq:Srecursion}) and (\ref{eq:recursionT}) respectively. Full details are given in Appendix \ref{sec:LinkWtihRL}. A similar correspondence between state actions and successive hidden states was also noted in \citep{weber2015reinforced} which explores the use of RL ideas in the context of variational inference. However, \citep{weber2015reinforced} exploits this correspondence to propose a backward in time recursion for the ELBO of the form (\ref{eq:backwardrecursionRL}) initialized at the time of the last observation of the SSM. When a new observation is collected at the next time step, this backward recursion has to be re-run which would lead to a computational time increasing linearly at each time step. In \cite{kim2020variational}, the links between RL and variational inference are also exploited. The likelihood of future points in an SSM is approximated using temporal difference learning \citep{sutton2018reinforcement}, but the proposed algorithm is not online. 

To perform online variational inference, \cite{zhao2020variational} proposes to use the decomposition (\ref{eq:logevidence}) of the log evidence $\ell_t(\theta)$ and lower bound each term $\log p_{\theta}(y_k|y^{k-1})$ appearing in the sum using 
\begin{equation}\label{eq:localELBO}
    \mathbb{E}_{q^{\phi}_{k}(x_{k-1},x_k)}\left[\log \frac{f_\theta(x_k|x_{k-1})g_\theta(y_k|x_k) p_\theta(x_{k-1}|y^{k-1})}{q^\phi_k(x_{k-1},x_k)}\right]\leq \log p_{\theta}(y_k|y^{k-1}).
\end{equation}
Unfortunately, the term on the l.h.s. of (\ref{eq:localELBO}) cannot be evaluated unbiasedly as it relies on the intractable filter $p_\theta(x_{k-1}|y^{k-1})$ so  \cite{zhao2020variational} approximates it by $p_\theta(x_{k-1}|y^{k-1}) \approx q^{\phi}_{k-1}(x_{k-1})$ to obtain the following \emph{Approximate} ELBO (AELBO) by summing over $k$: 
\begin{equation}\label{eq:approixmateELBO}
    \mathcal{\widetilde{L}}_{t}(\theta,\phi_{1:t})=\sum\nolimits_{k=1}^t \mathbb{E}_{q^{\phi}_{k}(x_{k-1},x_k)}\left[\log r^{\theta,\phi}_k(x_{k-1},x_k)\right].
\end{equation}
Additionally, \cite{zhao2020variational} makes the assumption $q^{\phi}_{k}(x_{k-1},x_k)=q^{\phi}_{k-1}(x_{k-1})q^{\phi}_{k}(x_k)$ and we will refer to (\ref{eq:approixmateELBO}) in this case as AELBO-1.
While \cite{gregor2018temporal} does not consider online learning, their objective is actually a generalization of \cite{zhao2020variational}, with $q^{\phi}_{k}(x_{k-1},x_k)=q^{\phi}_k(x_k)q^{\phi}_k(x_{k-1}|x_k)$, and we will refer to (\ref{eq:approixmateELBO}) in this case as AELBO-2.
It can be easily shown that AELBO-2 is only equal to the true ELBO given in Proposition \ref{prop:ELBOrecursion} in the (unrealistic) scenario where $q^{\phi}_{k}(x_{k})=p_\theta(x_k|y^k)$ for all $k$. Moreover, in both cases the term $p_\theta(x_{k-1}|y^{k-1})$ is replaced by $q^{\phi}_{k-1}(x_{k-1})$, causing a term involving $\theta$ to be ignored in gradient computation. The approach developed here can be thought of as a way to correct the approximate ELBOs computed in \citep{gregor2018temporal,zhao2020variational} in a principled manner, which takes into account the discrepancy between the filtering and approximate filtering distributions, and maintains the correct gradient dependencies in the computation graph. Finally \cite{zhao2020streaming} relies on the PF to do online variational inference. However the variational approximation of the filtering distribution is only implicit as its expression includes an intractable expectation and, like any other PF technique, its performance is expected to degrade significantly with the state dimension \cite{bengtsson2008curse}.

\section{Experiments} 
\label{sec: Experiments}

\subsection{Linear Gaussian State-Space Models}
We first consider a linear Gaussian SSM for which the filtering distributions can be computed using the KF and the RMLE recursion (\ref{eq:RMLE}) can be implemented exactly.  Here the model is defined as
\begin{align}
    f_\theta(x_t|x_{t-1}) = \mathcal{N}(x_t; Fx_{t-1}, U), \quad g_\theta(y_t|x_t) = \mathcal{N}(y_t; Gx_t, V),
\end{align}
where $F \in \mathbb{R}^{d_x \times d_x}$, $G \in \mathbb{R}^{d_y \times d_x}$, $U \in \mathbb{R}^{d_x \times d_x}$, $V \in \mathbb{R}^{d_y \times d_y}$, $\theta = \{F,G\}$. 
We let
\begin{equation}
    q_t^{\phi_t}(x_t) = \mathcal{N}\left(x_t; \mu_t, \text{diag}(\sigma^2_t)\right), \quad q_t^{\phi_t}(x_{t-1}|x_t) = \mathcal{N}\left(x_{t-1}; W_t x_t + b_t, \text{diag}(\tilde{\sigma}_t^2)\right),
\end{equation}
with $\phi_t = \{ \mu_t, \log \sigma_t, W_t, b_t, \log \tilde{\sigma_t} \}$. When $f_\theta(x_t | x_{t-1})$ is linear Gaussian and we use a Gaussian for $q_{t-1}^{\phi_{t-1}}(x_{t-1})$, we can select $q_t^{\phi_t}(x_{t-1}|x_t) \propto f_\theta(x_t | x_{t-1}) q_{t-1}^{\phi_{t-1}}(x_{t-1})$ as noted in \cite{pfrommer2022LVSSF}. In this experiment however, we aim to show our full method can converge to known ground truth values hence still fully parameterize $q_t^{\phi_t}(x_{t-1}|x_t)$ as well as setting the matrices $F,G,U,V$ to be diagonal, so that $p_\theta(x_t|y^t)$ and $p_\theta(x_{t-1}| y^{t-1},x_t)$ are in the variational family.

For $d_x=d_y=10$, we first demonstrate accurate state inference by learning $\phi_t$ at each time step whilst holding $\theta$ fixed at the true value. We represent $\hat{T}_t(x_t)$ non-parametrically using KRR. Full details for all experiments are given in Appendix \ref{sec:AppendixExperiments}.\footnote{ {Code available at \url{https://github.com/andrew-cr/online_var_fil}}} Figure \ref{fig:linear_q_phi} illustrates how, given extra computation, our variational approximation comes closer and closer to the ground truth, the accuracy being limited by the convergence of each inner stochastic gradient ascent procedure.
We then consider online learning of the parameters $F$ and $G$ using Algorithm \ref{alg:TheAlgorithm}, comparing our result to RMLE and a variation of  Algorithm \ref{alg:TheAlgorithm} using AELBO-1 and 2 (see Section \ref{sec:relatedwork}). Our methodology converges much closer to the analytic baseline (RMLE) than AELBO-2  \cite{gregor2018temporal} and exhibits less variance, even though the variational family is sufficiently expressive for AELBO-2 to learn the correct backward transition. In addition, we find that AELBO-1 \cite{zhao2020variational} did not produce reliable parameter estimates in this example, as it relies on a variational approximation that ignores the dependence between $x_{k-1}$ and $x_k$.
As expected, our method performs slightly worse than the analytic RMLE, as inevitably small errors will be introduced during stochastic optimization and regression.

\begin{figure}
\centering
\subfloat[]{
    \includegraphics[trim=0 20 0 0 ,width=0.4\textwidth]{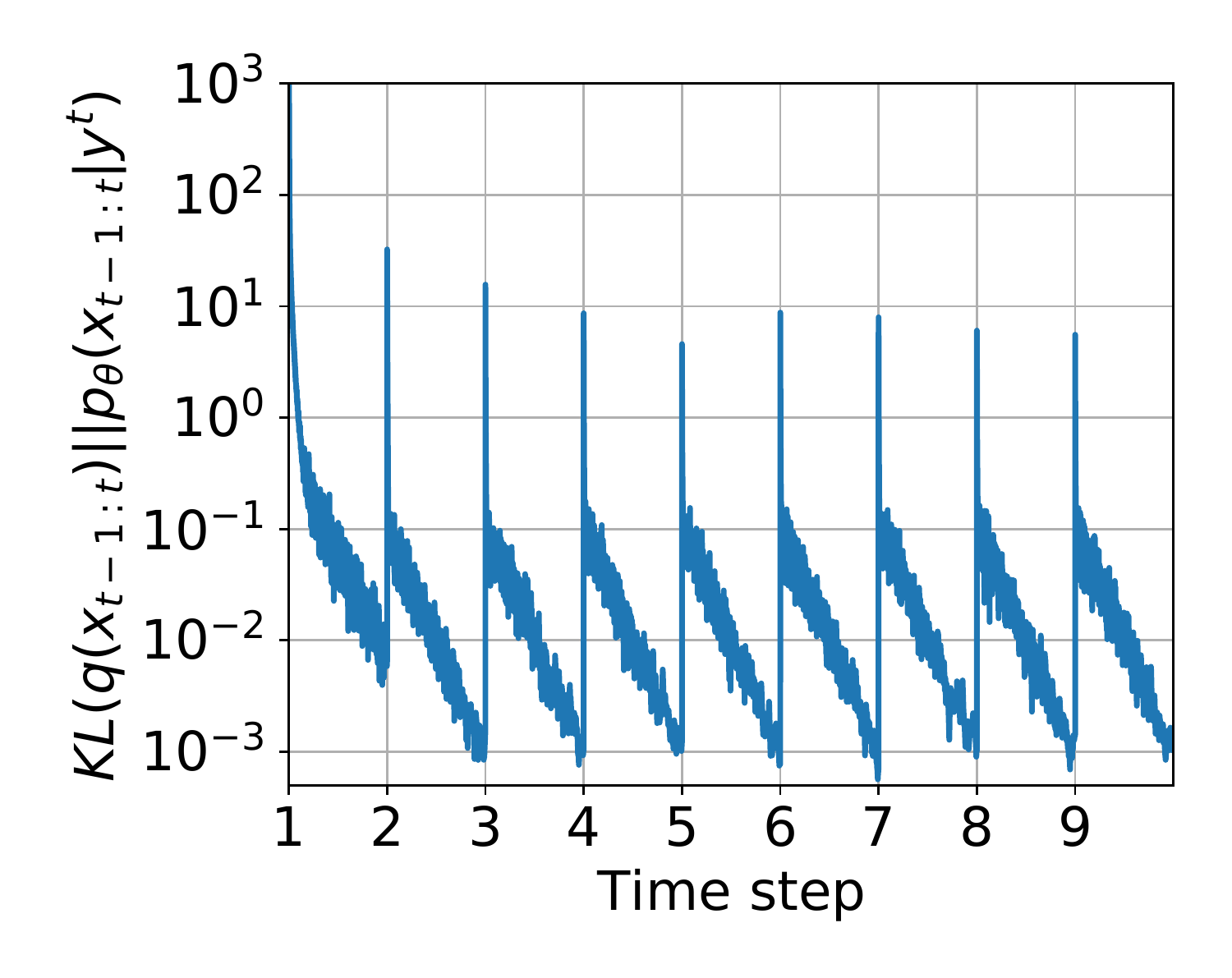}
    \label{fig:linear_q_phi}
}
\subfloat[]{
  \includegraphics[trim=0 20 0 0, width=0.59\textwidth]{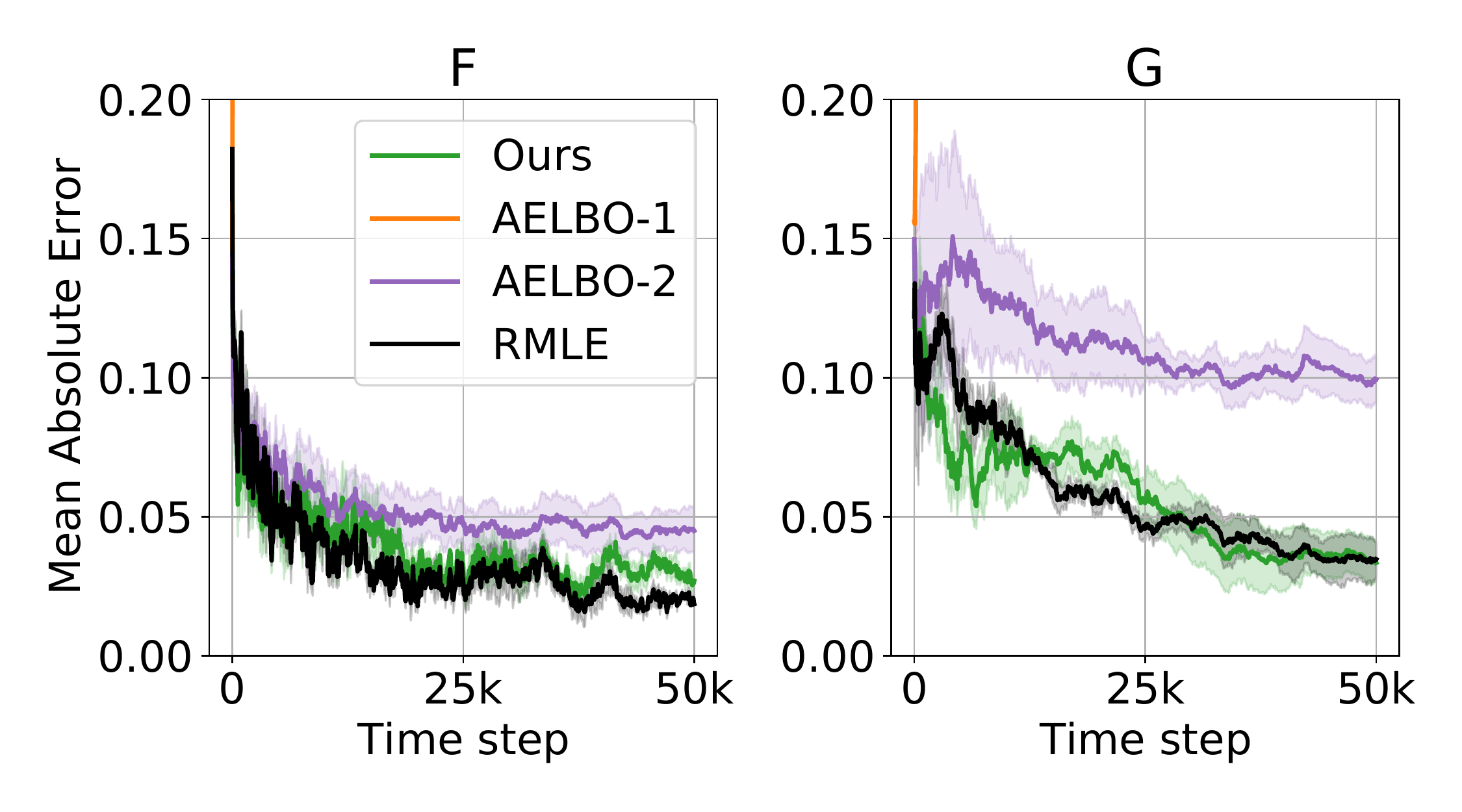}
  \label{fig:linear_model_learning}
}
    \caption{\textbf{(a)} $\text{KL}(q_t^{\phi_t}(x_{t-1},x_t) || \linebreak[1] p_\theta(x_{t-1},x_{t}|y^t))$ vs time step of the SSM. Between each time step, we plot the progress of the $\text{KL}$ over 5000 iterations of inner loop $\phi_t$ optimization. \textbf{(b)} Mean Absolute Error for model parameters $F$ (left) and $G$ (right) vs time step (AELBO-1 off the scale).}
\end{figure}

\subsection{Chaotic Recurrent Neural Network}

We next evaluate the performance of our algorithm for state estimation in non-linear, high-dimensional SSMs. We reproduce the Chaotic Recurrent Neural Network (CRNN) example in \cite{zhao2020streaming},
but with state dimension $d_x= 5, 20$, and $100$.
This non-linear model is an Euler approximation of the continuous-time
recurrent neural network dynamics 
\[
f(x_{t}|x_{t-1})=\mathcal{N}\left(x_t; x_{t-1}+\Delta \tau^{-1}\left(-x_{t-1}+\gamma W\tanh(x_{t-1})\right),U\right),
\]
and the observation model is linear with additive noise from
a Student's t-distribution. We compare our algorithm against ensemble
KF (EnKF), bootstrap PF (BPF), as well as variational methods using AELBO-1 and AELBO-2. We let $q_t^{\phi_t}(x_{t-1}|x_t) = \mathcal{N}(x_{t-1}; \text{MLP}_t^{\phi_t}(x_t), \text{diag}(\tilde{\sigma}_t^2))$ and $q_t^{\phi_t}(x_t) = \mathcal{N}\left(x_t; \mu_t, \text{diag}(\sigma^2_t)\right)$ 
where we use a 1-layer Multi-Layer Perceptron (MLP) with 100 neurons for each $q_t^{\phi_t}(x_{t-1}|x_t)$. 
We generate a dataset of length 100 using the same settings as \cite{zhao2020streaming}, and each algorithm is run 10 times to report the mean and standard deviation. We also match approximately the computational complexity for all methods. From Table \ref{tab:crnn_rmse_elbo}, we observe that the EnKF performs poorly on this non-linear model, while the PF performance degrades significantly with $d_x$, as expected. Among variational methods, AELBO-1 does not give as accurate state estimation, while AELBO-2 and our method achieve the lowest error in terms of RMSE. 
However, our method achieves the highest ELBO; i.e. lowest KL between the variational approximation and the true posterior since $\theta$ is fixed - an effect not represented using just the RMSE.
We confirm this is the case in Appendix \ref{sec:ChaotircRecurrentNeuralNetworkApdx} by comparing our variational filter means $\mu_t$ against the `ground truth' posterior mean for $d_x=5$ computed using PF with 10 million particles.
Furthermore, our method is also able to accurately estimate the true ELBO online. Figure \ref{fig: CTRNN ELBO learning} shows that our online estimate of the ELBO, RELBO (Section \ref{sec:theta_gradients}), is very close to the true ELBO, whereas AELBO-2 is biased and consistently overestimates it. Further, AELBO-1 is extremely loose meaning its approximation of the joint state posterior is very poor.

Using this CRNN problem, we also investigate the penalty we pay for our online method. At time $t$, we train $q_t^{\phi_t}(x_t),q_t^{\phi_t}(x_{t-1}|x_t)$ and hold $q_k^{\phi_k}(x_{k-1}|x_{k}), \, k < t$ fixed. This is because all information to learn the true factor $p_\theta(x_{k-1}|y^{k-1}, x_k)$ is available at time $k$ as it does not depend on future observations. However, when $p_\theta(x_{k-1}|y^{k-1}, x_k)$ is not in the variational family (as in this CRNN example), $q_k^{\phi_k}(x_{k-1}|x_k)$ will aim to be most accurate in the regions of $x_k$ in which it is typically evaluated. The evaluation distribution over $x_k$ does depend on future variational factors and so could shift over time. This may result in learning a different variational joint $q_t^{\phi_{1:t}}(x_{1:t})$ between when using our online ELBO and when just using the full offline ELBO (\ref{ELBOt}). We quantify this difference by training the same joint variational distribution using either our online ELBO or the offline ELBO (\ref{ELBOt}). Figure \ref{fig: CRNN offline comparison single} plots the final mean and standard deviations of the marginals of the trained $q_t^{\phi_{1:t}}(x_{1:t})$ in both cases. We see that these quantities are very close, suggesting this effect is not an issue on this example. This may be due to the evaluation distribution over $x_k$ not changing a significant enough amount to cause appreciable changes in the learned variational factors. We show this result holds over dimensions and seeds in Appendix \ref{sec:ChaotircRecurrentNeuralNetworkApdx}.

\begin{figure}[h]
\centering
\vspace{-1em} 
\subfloat[]{
    \includegraphics[trim=30 0 0 10 , width=0.48\textwidth]{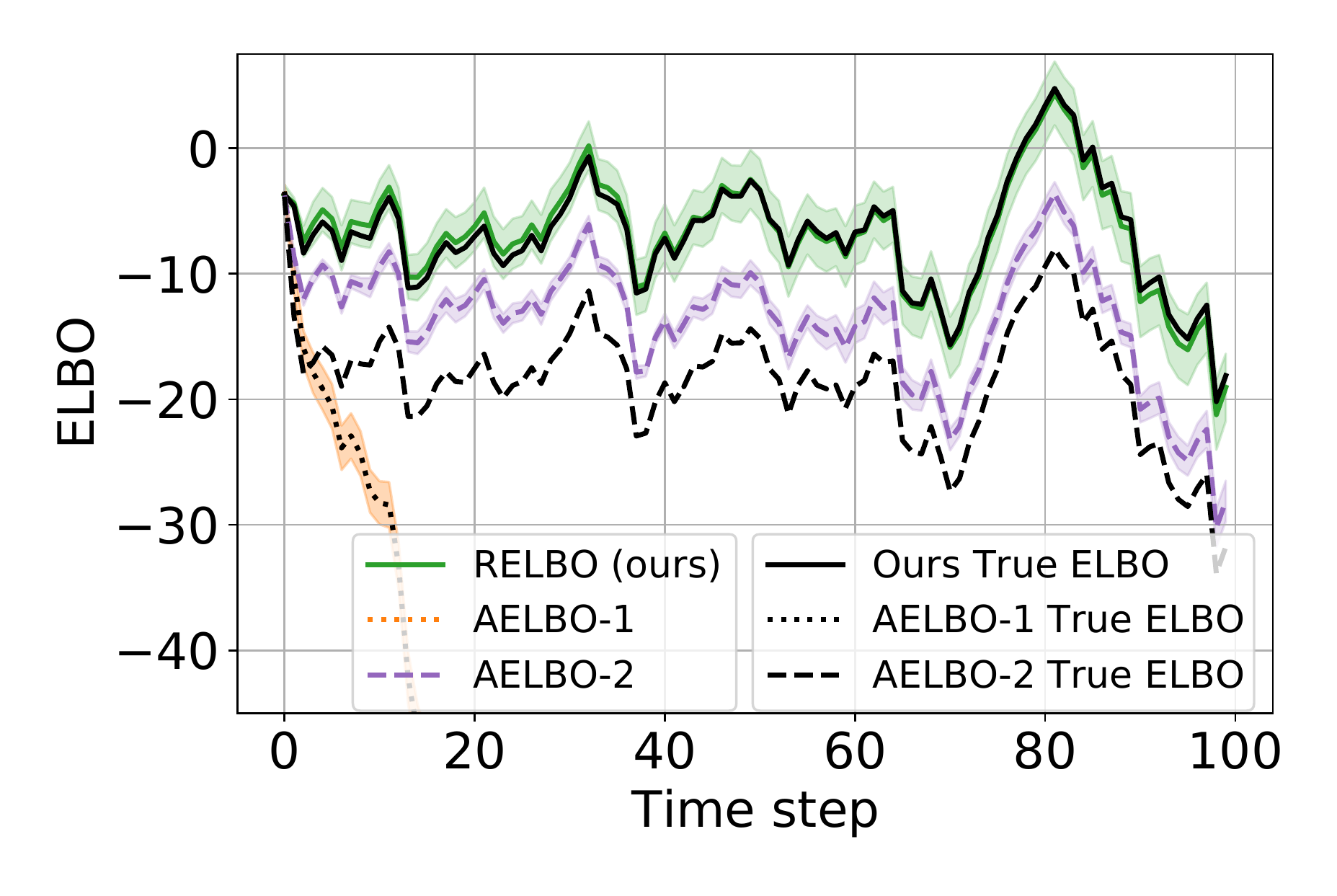}
    \label{fig: CTRNN ELBO learning}
}
\subfloat[]{
    \includegraphics[trim=0 -4 0 0, width=0.48\textwidth]{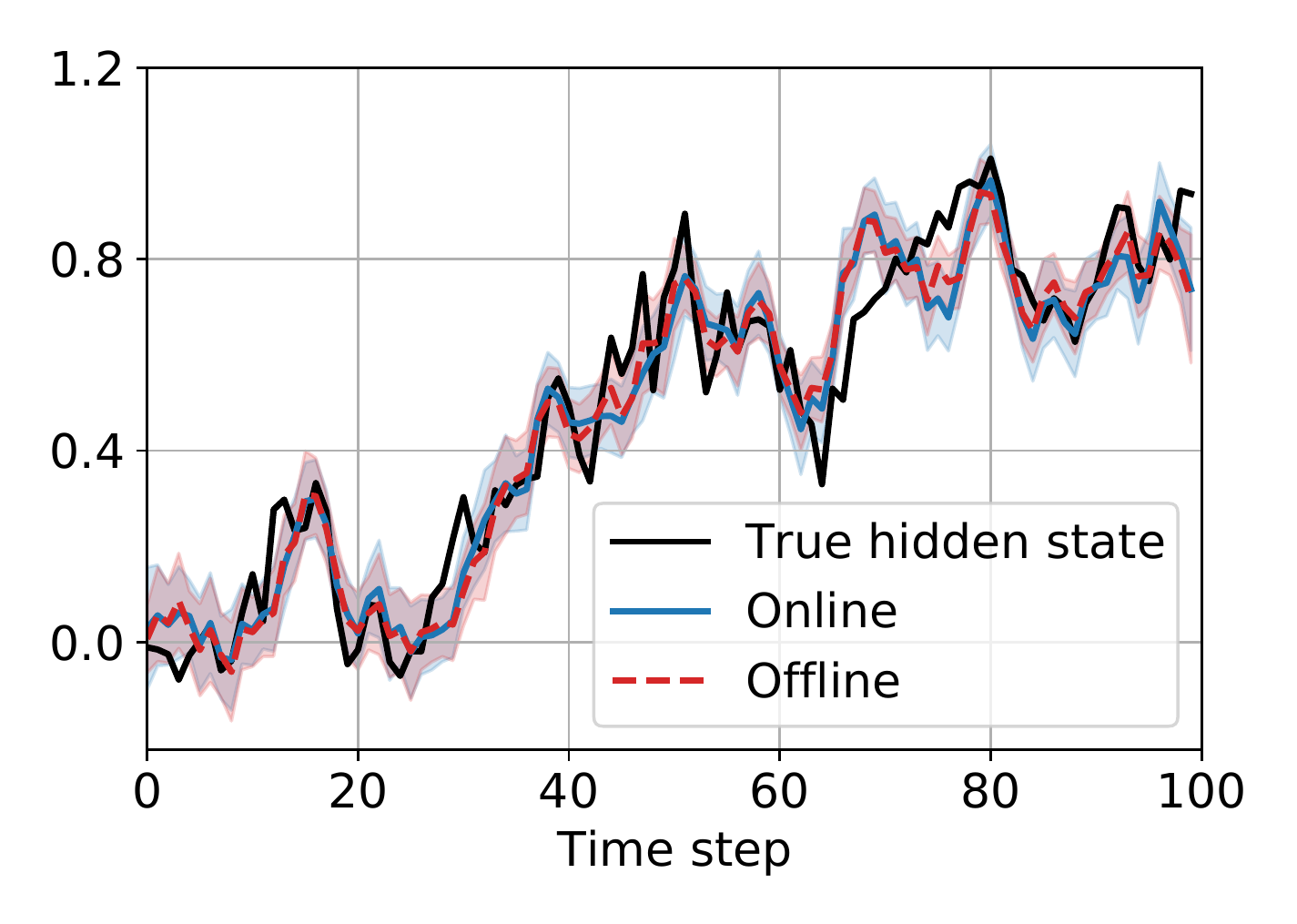}
    \label{fig: CRNN offline comparison single}
}
\vspace{-0.5em} 
    \caption{ \textbf{(a)} Estimates and true values of the ELBO on the Chaotic RNN task. RELBO uses KRR for $\hat{V}_t$ whilst for the other methods we use \cref{eq:approixmateELBO}. \textbf{(b)} Comparison between joint variational distributions trained online and offline on the CRNN task. The colored lines show the mean of $q_t^{\phi_{1:t}}(x_{1:t})$ whilst the shaded region represent $\pm 1$ std. The true hidden state is also shown in black.}
\end{figure}

\tabcolsep=0.12cm
\begin{table}
\centering
\small
    \caption{Root Mean Squared Error between filtering mean and true state and the average true ELBO for the 5 methods in varying dimensions on the Chaotic RNN task. }
    \vspace{1em} 
    \label{tab:crnn_rmse_elbo}
    \begin{tabular}{clccccc}
    \toprule 
    $d_{x}$ &  & EnKF & BPF & AELBO-1 & AELBO-2 & Ours\tabularnewline
    \midrule
    \midrule 
    \multirow{3}{*}{5} & Filter RMSE & 0.1450\textpm 0.0026 & \textbf{0.1026\textpm 0.0001} & 0.1284\textpm 0.0035 & 0.1035\textpm 0.0012 & 0.1032\textpm 0.0005\tabularnewline
    \cmidrule{2-7} \cmidrule{3-7} \cmidrule{4-7} \cmidrule{5-7} \cmidrule{6-7} \cmidrule{7-7} 
     & ELBO (nats) & - & - & -220.52\textpm 6.2768 & -30.944\textpm 2.2928 & \textbf{-15.845\textpm 1.7385}\tabularnewline
    \cmidrule{2-7} \cmidrule{3-7} \cmidrule{4-7} \cmidrule{5-7} \cmidrule{6-7} \cmidrule{7-7} 
     & Time per step & 1.0998 & 0.9268 & 1.5067 & 2.2270 & 2.6899\tabularnewline
    \midrule
    \midrule 
    \multirow{3}{*}{20} & Filter RMSE & 0.1541\textpm 0.0016 & 0.1092\textpm 0.0014 & 0.1355\textpm 0.0012 & 0.1086\textpm 0.0004 & \textbf{0.1082\textpm 0.0003}\tabularnewline
    \cmidrule{2-7} \cmidrule{3-7} \cmidrule{4-7} \cmidrule{5-7} \cmidrule{6-7} \cmidrule{7-7} 
     & ELBO (nats) & - & - & -928.80\textpm 10.463 & -393.68\textpm 3.9053 & \textbf{-340.36\textpm 3.9730}\tabularnewline
    \cmidrule{2-7} \cmidrule{3-7} \cmidrule{4-7} \cmidrule{5-7} \cmidrule{6-7} \cmidrule{7-7} 
     & Time per step & 5.1879 & 3.8932 & 2.3587 & 2.7000 & 3.5935\tabularnewline
    \midrule
    \midrule 
    \multirow{3}{*}{100} & Filter RMSE & 0.1571\textpm 0.0017 & 0.2493\textpm 0.0122 & 0.1239\textpm 0.0006 & 0.1070\textpm 0.0001 & \textbf{0.1068\textpm 0.0001}\tabularnewline
    \cmidrule{2-7} \cmidrule{3-7} \cmidrule{4-7} \cmidrule{5-7} \cmidrule{6-7} \cmidrule{7-7} 
     & ELBO (nats) & - & - & -4247.9\textpm 20.905 & -2069.7\textpm 11.814 & \textbf{-1794.7\textpm 5.4173}\tabularnewline
    \cmidrule{2-7} \cmidrule{3-7} \cmidrule{4-7} \cmidrule{5-7} \cmidrule{6-7} \cmidrule{7-7} 
     & Time per step & 6.4546 & 4.6184 & 3.2697 & 4.5539 & 5.9263\tabularnewline
    \bottomrule
    
    \end{tabular}
\end{table}

\subsection{Sequential Variational Auto-Encoder} \label{sec:SeqVAE}
We demonstrate the scalability of our method on a sequential VAE application. In this problem, an agent observes a long sequence of frames that could, for example, come from a robot traversing a new environment. The frames are encoded into a latent representation using a pre-trained decoder. The agent must then learn online the transition dynamics within this latent space using the single stream of input images. The SSM is defined by 
\begin{align}
    &f_\theta(x_t | x_{t-1}) = \mathcal{N}(x_t; \text{NN}^f_\theta(x_{t-1}), U), \quad &&g(y_t|x_t) = \mathcal{N}(y_t; \text{NN}^g(x_t), V),
\end{align}
where $d_x=32$, $\text{NN}_\theta^f$ is a residual MLP and $\text{NN}^g$ a convolutional neural network. $\text{NN}_\theta^f$ is learned online whilst $\text{NN}^g$ is fixed and is pre-trained on unordered images from similar environments using the standard VAE objective \citep{kingmaVAE}.
We perform this experiment on a video sequence from a DeepMind Lab environment \citep{beattie2016deepmind} (GNU GPL license). We use the same $q_t^{\phi_t}$ parameterization as for the CRNN but with a 2 hidden layer MLP with 64 neurons. KRR is used to learn $\hat{T}_t$ whereas we use an MLP for learning $\hat{S}_t$. We found that MLPs scale better than KRR as $d_\theta$ is high. Our online algorithm is run on a sequence of 4000 images after which we probe the quality of the learned $\text{NN}_\theta^f$. The results are shown in Figure \ref{fig:dmlab_rollout}. Before training, $\text{NN}_\theta^f$ predicts no meaningful change but after training it predicts movements the agent could realistically take. We quantify this further in Appendix \ref{sec:SeqVAEApdx} by showing that the approximate average log likelihood $\ell_t(\theta)/t$ computed using Monte Carlo increases through training, thereby confirming our method can successfully learn high-dimensional model parameters of the agent movement in a fully online fashion.
\begin{figure}
     \centering
     \begin{subfigure}[b]{1.0\textwidth}
         \centering
         \includegraphics[width=\textwidth]{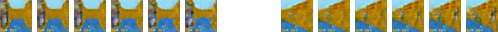}
         \caption{Before training}
     \end{subfigure}
     \begin{subfigure}[b]{1.0\textwidth}
         \centering
         \includegraphics[width=\textwidth]{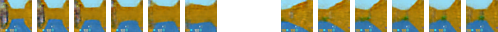}
         \caption{After training}
     \end{subfigure}
        \caption{Frames predicted by rolling out $\text{NN}_\theta^f$ from two different starting points, before and after training. Between each frame, 3 transition steps are taken. Before training, no meaningful change is predicted but after training $\text{NN}_\theta^f$ predicts plausible movements.}
        \label{fig:dmlab_rollout}
\end{figure}
\vspace{-0.15cm}
\section{Discussion} \label{sec:Discussion}

We have presented a novel online approach for variational state estimation and parameter learning. In our experiments, this methodology outperforms standard filtering approaches for high dimensional SSMs and is also able to learn online high dimensional model parameters of a neural network in a sequential VAE. However, it is not without its limitations. As with any stochastic variational inference technique, we can only obtain an accurate posterior approximation if the variational family is expressive enough and our stochastic gradient method finds a reasonable minimum. We also need to use flexible function approximators to keep the bias in our gradient estimates small. Finally, although our method is online, it can be quite computationally expensive in absolute terms as it requires solving an optimization problem and solving a regression task for each time step. 

To reduce this computational cost, one can amortize the optimization cost for $\phi$ by learning a network taking a representation of observations up to now, $y^t$, and outputting $\smash{q_t^{\phi}}$ as illustrated in Appendix \ref{sec:AppendixAmortization}. Further work will investigate ways to also amortize the cost of function regression across time through a meta-learning strategy. From a theoretical point of view, it would be useful to establish conditions under which the proposed variational filter is exponentially stable and to study the asymptotic properties of the parameter learning procedure.

\newpage
\begin{ack}
The authors are grateful to Adrien Corenflos, Desi Ivanova and James Thornton for their comments. Andrew Campbell acknowledges support from the EPSRC CDT in Modern Statistics and Statistical Machine Learning (EP/S023151/1). Arnaud Doucet is partly supported by the EPSRC grant EP/R034710/1. He also acknowledges support of the UK Defence Science and Technology Laboratory (DSTL) and EPSRC under grant EP/R013616/1. This is part of the collaboration between US DOD, UK MOD and UK EPSRC under the Multidisciplinary University Research Initiative.  
\end{ack}

\bibliographystyle{apalike}
\bibliography{references}

\newpage

\appendix
\part{Appendix}
\parttoc

The appendix is organized as follows. In Section \ref{sec:ApdxMethodologicalDetails} we cover in more detail some subtleties of our proposed methodology. In \ref{sec:ObjectiveForRecursionFitting}, we motivate our choice of objectives for fitting our recursive gradient estimators, \ref{sec:AppendixAlgorithm} gives a full description of our algorithm and \ref{sec:LinkWtihRL} details the link between the recursions used in RL and those presented here. Section \ref{sec:AppendixExperiments} gives experimental details as well as some further results for our Linear Gaussian (\ref{sec:LinearGaussianExperimentApdx}), Chaotic Recurrent Neural Network (\ref{sec:ChaotircRecurrentNeuralNetworkApdx}) and Sequential VAE (\ref{sec:SeqVAEApdx}) experiments. The proofs for all our propositions are given in Section \ref{sec:Proofs}. In Section \ref{sec:AppendixAmortization}, we discuss methods for amortizing the cost of learning $\phi_t$ over time. We present possible architecture choices in \ref{sec:ApdxAmortizaitonArchitecture}, alternative objectives in \ref{sec:ApdxAmortizationUfunction}, \ref{sec:ApdxSemiAmortized}, and notes on gradient computations in \ref{sec:ApdxAmortizationGradients}. Finally, we discuss the broader impact of our research in Section \ref{sec:BroaderImpact}.

\newpage

\section{Methodological Details} \label{sec:ApdxMethodologicalDetails}
\subsection{Objective for Recursive Fitting} \label{sec:ObjectiveForRecursionFitting}
Before specializing to our application, we first state a standard regression result. The following regression problem 
\begin{equation}
\label{eq:outerExpectationRegression}
    \underset{h}{\text{min}} \quad \E_{\rho(x,y)} \left[ \left\lVert h(x) - k(y, x)\right\rVert_2^2\right]
\end{equation}
has solution $h(x) = \E_{\rho(y|x)}\left[ k(y, x) \right]$. Therefore, we can estimate $ \E_{\rho(y|x)}\left[ k(y, x) \right]$ by minimizing the following empirical version of the $L^2$ loss (\ref{eq:outerExpectationRegression}) (or a regularized version of it)  
\begin{equation}
    \underset{h \in \mathcal{F}}{\text{min}}\quad  \frac{1}{N}\sum_{i=1}^{N} \lVert h(x^i) - k(y^i, x^i)\rVert_2^2,\quad\text{where~}x^i,y^i \overset{\textup{i.i.d.}}\sim \rho(x,y)
\end{equation}
over a flexible function class $ \mathcal{F}$. 

We note that if we make the following substitutions, then this solution exactly corresponds to the form of recursion (\ref{eq:Srecursion}) with our approximation $\hat{S}_{t+1}$ substituted for $S_{t+1}^{\theta, \phi_{1:t+1}}$:
\begin{align}
    &x = x_{t+1}, \quad \rho(x) = q_{t+1}^{\phi_{t+1}}(x_{t+1}), \quad y = x_t, \quad \rho(y|x) = q_{t+1}^{\phi_{t+1}}(x_t|x_{t+1}),\\
    &h(x) = \hat{S}_{t+1}(x_{t+1}), \quad k(y,x) = \hat{S}_t(x_t) + s_{t+1}^\theta(x_t, x_{t+1}).
\end{align}

Similarly for fitting $\hat{T}_{t+1}$ we can make the following substitutions such that the regression solution corresponds to (\ref{eq:recursionT}).
\begin{align}
    &x = x_{t+1}, \quad \rho(x) = q_{t+1}^{\phi_{t+1}}(x_{t+1}), \quad y = \epsilon_t, \quad \rho(y|x) = \lambda(\epsilon_t), \quad h(x) = \hat{T}_{t+1}(x_{t+1}),\\
    & k(y,x) = \hat{T}_t(x_t(\phi_{t+1}; \epsilon_t, x_{t+1})) \frac{\partial x_t(\phi_{t+1}; \epsilon_t, x_{t+1})}{\partial x_{t+1}} + \nabla_{x_{t+1}} r_{t+1}^{\theta, \phi_{t:t+1}}(x_t(\phi_{t+1}; \epsilon_t, x_{t+1}), x_{t+1}).
\end{align}

We note that the choice of distribution $\rho(x)$ is technically arbitrary, however, in practice, it will decide the region of space that our function approximation is most accurate. Therefore, we would like it to be close to the `test time' distribution - the distribution of points that the approximators are evaluated at during the next time step.
For $\hat{T}_{t+1}$, the test time distribution depends on $\phi_{t+2}$ which changes during optimization. This gives a series of test time distributions given by
\begin{equation}
    \label{eq:qtp2back1smoothing}
    \int q_{t+2}^{\phi_{t+2}}(x_{t+2}) q_{t+2}^{\phi_{t+2}}(x_{t+1}|x_{t+2}) dx_{t+2}.
\end{equation}
Assuming an accurate variational approximation, this will approach the following one-step smoothing distribution: $p_\theta(x_{t+1} | y^{t+2})$. Our best approximation to this available at time $t+1$ is $q_{t+1}^{\phi_{t+1}}(x_{t+1})$ which approximates $p_\theta(x_{t+1} | y^{t+1})$. $\hat{S}_{t+1}$ is evaluated at the end of $\phi_{t+2}$ optimization so it has only two test time distributions, (\ref{eq:qtp2back1smoothing}) with the final value of $\phi_{t+2}$ and $q_{t+1}^{\phi_{t+1}}(x_{t+1})$. Therefore, for both $\hat{T}_{t+1}$ and $\hat{S}_{t+1}$, we set $\rho(x) = q_{t+1}^{\phi_{t+1}}(x_{t+1})$. We found this to perform well in our experiments and so strategies to mitigate the distribution shift between training and test time were not needed. However, if necessary, we could artificially inflate the entropy of the training distribution to cover a wider region of the input space.

\subsection{Algorithm} \label{sec:AppendixAlgorithm}
The full detailed description of Algorithm \ref{alg:TheAlgorithm} is shown in Algorithm \ref{alg:DetailedAlgorithm}.

Here, we have assumed the variational family is Gaussian and so we detail the reparameterization trick. We parameterize the standard deviation through $\log \sigma_t$ to ensure $\sigma_t>0$.

We present both options for regression, neural networks and KRR. The exposition for KRR is as in \cite{li2020amortised}. When we use neural networks, we use $\varpi_t$ to represent the parameters of the $\hat{T}_t$ network and $\psi_t$ to represent the parameters of the $\hat{S}_t$ network. We describe L2 regression for the neural networks in Algorithm \ref{alg:DetailedAlgorithm} although other losses are possible. KRR is based on a regularized L2 loss with regularization parameter $\lambda > 0$. It requires a kernel $k(x^i_t, x^j_t)$ that takes two input vectors and outputs a scalar similarity between them. In our experiments we use the Mat\'ern kernel.

\LinesNumbered
\begin{algorithm}[t]
\SetAlgoNoLine
\DontPrintSemicolon
{\small \For{$t = 1, \dots, T$} {
Initialize $\phi_t$ e.g. $\phi_t \leftarrow \phi_{t-1}$\;
{\small \color{gray}\tcc{Update $\phi_t$ using $M$ stochastic gradient steps}}
    \For{$m=1, \dots, M$}{
        {\small \color{gray}\tcc{Split out variational parameters}}
        $[\mu_t, \log \sigma_t, \tilde{\phi}_t] \leftarrow \phi_t$\; \label{algLine:SampleXstart}
        {\small \color{gray}\tcc{Sample $x_{t-1}$ and $x_t$ using the reparameterization trick}}
        Sample $\epsilon_t^i \sim \mathcal{N}(0, \mathbb{I}_{d_x})$ \,\, for $i=1, \dots, N$\;
        $x_t^i = \mu_t + \exp(\log \sigma_t) \epsilon_t^i$ \,\, for $i=1, \dots, N$ \quad {\small \color{gray}\tcc{Element wise}}
        $[\tilde{\mu}_{t}^i, \log \tilde{\sigma}_{t}^i] = Q^{\tilde{\phi}_t}(x_t^i)$ \,\, for $i=1, \dots, N$\; 
        {\small \color{gray}\tcc{$Q^{\tilde{\phi}_t}$ is a function giving $q_t^{\phi_t}(x_{t-1}|x_t)$ statistics e.g. MLP}}
        Sample $\tilde{\epsilon}_t^i \sim \mathcal{N}(0, \mathbb{I}_{d_x})$ \,\, for $i=1, \dots, N$\;
        $x_{t-1}^i = \tilde{\mu}_t + \exp(\log \tilde{\sigma}_t) \tilde{\epsilon}_t^i$ \,\, for $i=1, \dots, N$ \quad {\small \color{gray}\tcc{Element wise}} \label{algLine:SampleXend}

        $\phi_t \leftarrow \phi_t + \gamma_m \frac{1}{N} \sum_{i=1}^N \{\hat{T}_{t-1}(x_{t-1}^i) \frac{\textup{d} x_{t-1}^i}{\textup{d} \phi_{t}} + \nabla_{\phi_t} r_t(x_{t-1}^i, x_t^i) \}$\;
    }
\BlankLine\;
{\small \color{gray} \tcc{Update $\hat{T}_t(x_t)$ and $\hat{S}_t(x_t)$ as in Section \ref{sec:theta_gradients}}}
{\small \color{gray} \tcc{Generate training datasets}}
Sample $x_{t-1}^i, x_t^i \sim q_t^{\phi_t}(x_{t-1}, x_t)$ as lines [\ref{algLine:SampleXstart}] to [\ref{algLine:SampleXend}] \,\, for $i = 1, \dots, P$\;
$\mathcal{D}_T = \left\{x_t^i,  \text{T-target}^i \right\}_{i=1}^P$ with $\text{T-target}^i = \hat{T}_{t-1}(x_{t-1}^i) \frac{\partial x_{t-1}^i}{\partial x_t^i} + \nabla_{x_t} r_t(x_{t-1}^i, x_t^i)$\;
$\mathcal{D}_S = \left\{x_t^i, \text{S-target}^i \right\}_{i=1}^P$ with $\text{S-target}^i = \hat{S}_{t-1}(x_{t-1}^i) + s_t^{\theta_{t-1}}(x_{t-1}^i, x_t^i)$\;

{\small \color{gray} \tcc{Update function approximators}}
\uIf{Regression using Neural Nets}{
\For {$j=1, \dots, J$}{
    $\mathcal{I} \leftarrow \text{minibatch sample from } \{1, \dots, P\}$\; 
    $\varpi_t \leftarrow \varpi_t + \gamma_j \nabla_{\varpi_t} \frac{1}{|\mathcal{I}|} \sum_{i \in \mathcal{I}} \| \hat{T}_t^{\varpi_t}(x_t^i) - \text{T-target}^i \|_2^2$\;
    $\psi_t \leftarrow \psi_t + \gamma_j \nabla_{\psi_t} \frac{1}{|\mathcal{I}|} \sum_{i \in \mathcal{I}} \| \hat{S}_t^{\psi_t}(x_t^i) - \text{S-target}^i \|_2^2$\;
}
}
\ElseIf{Regression using KRR}{
Let $\hat{T}_t(x_t^*) = R_{\mathcal{D}_T}(K_{\mathcal{D}_T} + P\lambda \mathbb{I}_P)^{-1} k^*_{\mathcal{D}_T}$\;
Let $\hat{S}_t(x_t^*) = R_{\mathcal{D}_S}(K_{\mathcal{D}_S} + P\lambda \mathbb{I}_P)^{-1} k^*_{\mathcal{D}_S}$\;
with\;
\quad $R_{\mathcal{D}} = [ \text{target}^1, \dots, \text{target}^P ] \in \mathbb{R}^{d_x \times P}$\;
\quad $K_{\mathcal{D}} \in \mathbb{R}^{P \times P}, \,\, (K_{\mathcal{D}})_{ij} = k(x^i_t, x^j_t)$\;
\quad $k^*_{\mathcal{D}} \in \mathbb{R}^P, \,\, (k^*_{\mathcal{D}})_i = k(x^i_t, x_t^*)$

}

\BlankLine\;
{\small \color{gray} \tcc{Update $\theta$ using a stochastic gradient step}}
Sample $x_{t-1}^i, x_t^i \sim q^{\phi_t}_t(x_{t-1}, x_t), \quad \tilde{x}_{t-1}^i \sim q^{\phi_{t-1}}_{t-1}(x_{t-1})$ for $i=1,...,N$\;
$\theta_t \leftarrow \theta_{t-1} + \eta_t \frac{1}{N} \sum_{i=1}^N \{ \hat{S}_{t-1}(x_{t-1}^i) + s_t^{\theta_{t-1}}(x_{t-1}^i, x_t^i)  - \hat{S}_{t-1}(\tilde{x}_{t-1}^i) \}$\;
}
}
 \caption{Online Variational Filtering and Parameter Learning - Full Algorithm Description}
\label{alg:DetailedAlgorithm}
\end{algorithm}

\subsection{Link with Reinforcement Learning}
\label{sec:LinkWtihRL}
The forward recursions in our method bear some similarity to the Bellman recursion present in RL. This is due to both relying on dynamic programming. We make explicit the relationship between our gradient recursions and the Bellman recursion in this section.

We first detail the standard RL framework and its Bellman recursion. The total expected reward we would like to optimize is 
\begin{equation}
    J(\phi) = \mathbb{E}_{\tau \sim p_\phi} \left[ \sum_{t=1}^T r(s_t, a_t)\right],
\end{equation}
where $s_t$ is the state at time $t$, $a_t$ is the action taken at time $t$ and $r(s_t, a_t)$ is the reward for being in state $s_t$ and taking action $a_t$. This expectation is taken with respect to the trajectory distribution which is dependent on the policy $\pi_\phi$
\begin{equation}
    p_\phi(\tau) = P(s_1) \pi_\phi(a_1|s_1) \prod_{t=2}^T P(s_t|s_{t-1}, a_{t-1}) \pi_\phi(a_t|s_t).
    \label{traj}
\end{equation}
The value function is then defined as the expected sum of future rewards when starting in state $s_t$ under policy $\pi_\phi$
\begin{equation}
    V^{\text{RL}}_t(s_t) := \mathbb{E} \left[ \sum_{k=t}^T r(s_k, a_k) \right].
\end{equation}

This value function satisfies the following Bellman recursion
\begin{equation}
    V^{\text{RL}}_t(s_t) = \mathbb{E}_{s_{t+1}, a_t \sim P(s_{t+1}|s_t, a_t) \pi_\phi(a_t|s_t)} \left[ r(s_t, a_t) + V^{\text{RL}}_{t+1}(s_{t+1})\right].
\end{equation}
The total expected reward $J(\phi)$ is then just the expected value of $V^{\text{RL}}_1$ taken with respect to the first state distribution
\begin{equation}
    J(\phi) = \mathbb{E}_{s_1 \sim P(s_1)} \left[ V^{\text{RL}}_1(s_1)\right].
\end{equation}
For our application, we would like to instead have a \emph{forward} recursion. A natural forward recursion appears when we consider an \emph{anti-causal} graphical model for RL, where $s_t$ depends on $s_{t+1}$ and $a_{t+1}$; i.e. we consider the following \emph{reverse}-time decomposition of the trajectory distribution
\begin{equation}
    p_\phi(\tau) = P(s_T) \pi_\phi(a_T|s_T) \prod_{t=T-1}^{1} P(s_t|s_{t+1}, a_{t+1}) \pi_\phi(a_t|s_t).
\end{equation}
We define a new value function which is the sum of previous rewards
\begin{equation}
    V_{t}(s_{t}) := \mathbb{E} \left[ \sum_{k=1}^{t} r(s_k, a_k)\right].
\end{equation}
It follows a corresponding forward Bellman-type recursion
\begin{equation}
    V_{t+1}(s_{t+1}) = \mathbb{E}_{\pi_\phi(a_{t+1}|s_{t+1}) P(s_t| s_{t+1}, a_{t+1})} \left[ r(s_{t+1}, a_{t+1}) + V_t(s_t)\right].
\end{equation}
$J(\phi)$ is then now the expected value of $V_T$ taken with respect to the final state distribution
\begin{equation}
    J(\phi) = \mathbb{E}_{s_T \sim P(s_T)}\left[ V_T(s_T)\right].
\end{equation}
This forward Bellman recursion is non-standard in the literature but is useful when we adapt it for our application. We define $s_t = x_t$, $a_t = x_{t-1} \sim q^{\phi}_{t}(x_{t-1}|x_{t})$ and $P(s_t|s_{t+1}, a_{t+1}) = \delta (s_t = a_{t+1})$. The `reward' is defined as
\begin{equation}
    r(s_t, a_t) = r(x_t, x_{t-1}) =  \log \frac{f(x_t|x_{t-1}) g(y_t|x_t) q_{t-1}^{\phi_{t-1}}(x_{t-1})}{q_t^{\phi_t}(x_t) q_t^{\phi_t}(x_{t-1}|x_t)},
\end{equation}
\begin{equation}
    r(s_1, a_1) = r(x_1, x_0) =  \log \frac{\mu(x_1) g(y_1|x_1)}{q_1^{\phi_1}(x_1)},
\end{equation}
where we have suppressed $\theta$ from the notation for conciseness.
Note $a_1 = x_0$ has no meaning here. The `policy' is defined as the backward kernel
\begin{equation}
    \pi_\phi(a_t|s_t) = q_t^{\phi_t}(x_{t-1}|x_t).
\end{equation}
With these definitions, the trajectory distribution is
\begin{equation}
    p_\phi(\tau) = q_T^{\phi_T}(x_T)\prod_{t=T}^1 q_t^{\phi_t}(x_{t-1}|x_t).
\end{equation}
(Since $x_0$ has no meaning in our application, the final $q_1^{\phi_1}(x_0|x_1)$ distribution that appears has no significance.)
With this formulation, the sum of rewards now corresponds to the ELBO which we would like to maximize with respect to $\phi$
\begin{equation}
    \mathcal{L}_{T} = \mathbb{E}_{p_\phi(\tau)} \left[ \sum_{t=1}^T r(s_t, a_t)\right].
\end{equation}
Just as in our anti-causal RL example, this can be broken down into a value function that summarizes previous rewards
\begin{equation}
    V_{t+1}(x_{t+1}) = \mathbb{E}\left[ \sum_{k=1}^{t+1} r(s_k, a_k)\right] = \mathbb{E}_{q_{t+1}^{\phi_{1:t+1}}(x_{1:t}|x_{t+1})} \left[ \log \frac{p_{\theta}(x_{1:t+1}, y^{t+1})}{q_{t+1}^{\phi_{1:t+1}}(x_{1:t+1})}\right],
\end{equation}
\begin{equation}
    \mathcal{L}_{T} = \mathbb{E}_{q_T^{\phi_T}(x_T)}\left[ V_T(x_T)\right].
\end{equation}
This follows a forward Bellman recursion (equation (\ref{eq:recursionV}) in the main text).
\begin{equation}
    V_{t+1}(x_{t+1}) = \mathbb{E}_{q_{t+1}^{\phi_{t+1}}(x_t|x_{t+1})} \left[ \log \frac{f(x_{t+1}|x_t)g(y_{t+1}|x_{t+1})q_t^{\phi_t}(x_t)}{q_{t+1}^{\phi_{t+1}}(x_{t+1}) q_{t+1}^{\phi_{t+1}}(x_t|x_{t+1})} + V_t(x_t)\right].
\end{equation}
Since we would like to optimize the ELBO rather than just evaluate it, we do not make use of $V_{t}(x_{t})$ directly. We instead differentiate this forward in time Bellman recursion to obtain our gradient recursions. To obtain equation (\ref{eq:Srecursion}) in the paper we differentiate with respect to $\theta$. To obtain equation (\ref{eq:recursionT}) we differentiate with respect to $x_{t}$, we then use $\frac{\partial}{\partial x_{t}} V_{t}(x_{t})$ to get an equation for $\nabla_{\phi_{t+1}} V_{t+1}(x_{t+1})$.\\

Our approach here is complementary to that of \citep{levine2018reinforcement, fellows2019virel} but differs in the fact we focus on forward in time recursions allowing an online optimization of the ELBO. \cite{levine2018reinforcement} and subsequent work focus on fitting RL into a probabilistic context whereas we take ideas from RL (recursive function estimation) to enable online inference. We note that \cite{weber2015reinforced} also define suitable rewards to fit probabilistic inference into an RL framework but again only focus on backward Bellman recursions.\\

\section{Experiment Details} \label{sec:AppendixExperiments}
\subsection{Linear Gaussian Experiment} \label{sec:LinearGaussianExperimentApdx}
For both experiments, we randomly initialize $F$ and $G$ to have eigenvalues between 0.5 and 1.0. When learning $\phi$ we set the diagonals of $U$ and $V$ to be $0.1^2$. We use a learning rate of 0.01 over 5000 iterations for each time step. We decay the learning rate by 0.999 at every inner training step. For the initial time point we use a learning rate of 0.1 with the same number of iterations and decay because the $\phi$ parameters start far away from the optimum but for the proceeding time points we initialize at the previous solution hence they are already close to the local optimum. We represent $\hat{T}_t$ using KRR, and use 500 data points to perform the fitting at each time step. We set the regularization parameter to 0.1. We use an RBF kernel with a bandwidth learned by minimizing mean squared error on an additional validation dataset. 

For learning $\phi$ and $\theta$ jointly, we set the diagonals of $V$ to be $0.25^2$. We compare different training methods against the RMLE baseline under the same settings. To learn $\phi$, we use a learning rate of 0.01 over 500 iterations for each time step with a learning rate decay of 0.991. We represent $\hat{T}_t$ using KRR, and use 512 data points to perform the fitting at each time step. We set the regularization parameter to 0.01. To learn $\theta$, we use a learning rate of 1e-2 and a Robbins-Monro type learning rate decay. We represent $\hat{S}_t$ using KRR with 1000 data points and regularization parameter 1e-4. The experiments were run on an internal CPU cluster with Intel Xeon E5-2690 v4 CPU.

\subsection{Chaotic Recurrent Neural Network} \label{sec:ChaotircRecurrentNeuralNetworkApdx}
We follow closely \cite{zhao2020streaming} using the same parameter settings for data generation. We use 1 million particles for EnKF and BPF for dimension $d_{x}=5$ and $20$, and 250000 particles for dimension $d_{x}=100$ in order to match the computation cost. For variational methods, we train for 500 iterations at each timestep with minibatch size 10 and learning rate 1e-2, and we use a single-layer MLP with 100 neurons to parameterize each $q_{t}^{\phi_t}(x_{t-1}|x_{t})$ for AELBO-2 and our method. The function $\hat{T}_t$ is represented using KRR with 100 samples for $d_{x}=5$ and 250 samples for $d_{x}=20$ and $100$. The regularization strength is fixed to be 0.1 while the kernel bandwidth is learned at each timestep on a validation dataset for 25 iterations with minibatch size 10 and learning rate 1e-2. The extra time for optimizing the kernel parameter is included in the presented runtime results. The experiments were run on an internal CPU cluster with Intel Xeon E5-2690 v4 CPU.

For $d_x=5$, we further confirm the gain in ELBO by comparing the variational means of $q_{t}^{\phi_t}(x_{t})$ and $q_{t}^{\phi_t}(x_{t-1})=\int q_{t}^{\phi_t}(x_{t})q_{t}^{\phi_t}(x_{t-1}|x_{t})\textup{d}x_t$ against the `ground truth' posterior means of $p(x_{t}|y^{t})$ and $p(x_{t-1}|y^{t})$ computed using PF with 10 million particles. As shown in Table \ref{tab:crnn_posterior_rmse}, our method attains a significantly lower error in terms of both metrics.\\

\begin{table}
\centering
\caption{Root Mean Squared Error between the `ground truth' posterior mean
and variational mean estimates for the filtering distribution $p(x_{t}|y^{t})$
and one-step smoothing distribution $p(x_{t-1}|y^{t})$ over 10 runs. }
\vspace{1em}
\begin{tabular}{ccccc}
\toprule 
$d_{x}$ &  & AELBO-1 & AELBO-2 & Ours\tabularnewline
\midrule
\midrule 
\multirow{2}{*}{5} & Filter mean RMSE & 0.0644\textpm 0.0037 & 0.0155\textpm 0.0006 & \textbf{0.0128\textpm 0.0007}\tabularnewline
\cmidrule{2-5} \cmidrule{3-5} \cmidrule{4-5} \cmidrule{5-5} 
 & 1-step mean RMSE & - & 0.0241\textpm 0.0008 & \textbf{0.0202\textpm 0.0009}\tabularnewline
\bottomrule
\label{tab:crnn_posterior_rmse}
\end{tabular}

\vspace{-0.2cm}
\end{table}

For our comparison with the offline ELBO, we verify our conclusion holds by considering a range of seeds and examining all dimensions of the variational joint. We plot the two joint variational distributions alongside the true hidden state in Figure \ref{fig:fullELBOcomparison} for 3 seeds and 5 dimensions. We see that in all cases the joint variational distributions trained with our online ELBO and the offline ELBO match very closely.

We also note that for longer time intervals, using the offline ELBO objective results in training instabilities due to the necessity of rolling out the entire backward joint variational distribution. To train for 100 time steps, a `warm start' was needed whereby we first train for a small number of training iterations on each of a series of intermediary objectives. If the total number of time steps is $t$, then there are $t$ intermediary objectives each being an offline ELBO with $\tau$ terms for $\tau=1, 2, \dots, t$. After the warm start, many gradient updates are taken using the offline ELBO corresponding to the full $t$ steps.\\

\begin{figure}
    \centering
    \includegraphics[width=\textwidth]{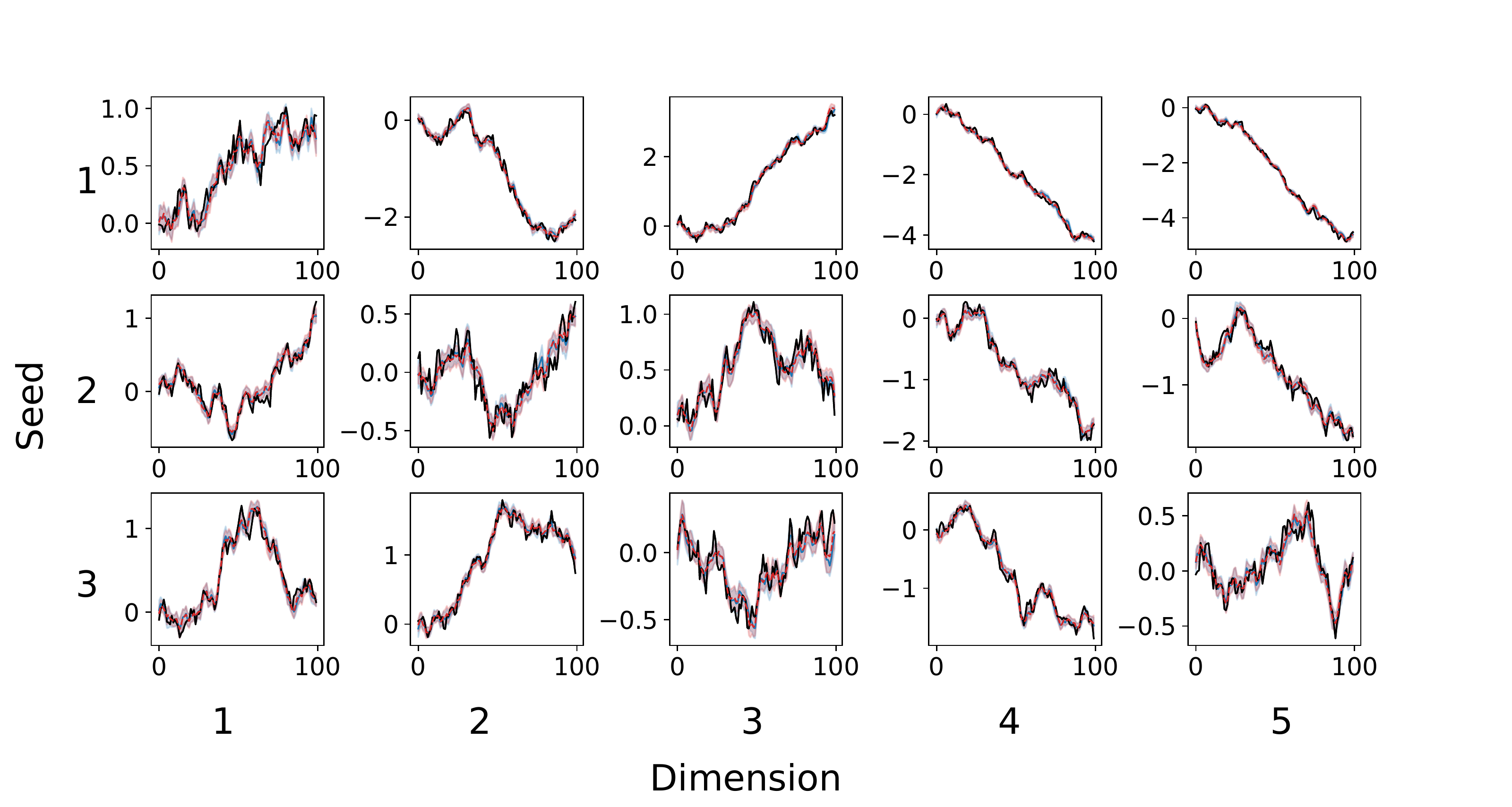}
    \caption{Comparison between joint variational distributions trained {\color{red} offline} (red dashed) versus {\color{blue} online} (blue full). The data is generated as in the Chaotic Recurrent Neural Network example, for 10 time steps and is shown in black. Colored lines show variational means through time whilst the shaded regions represent $\pm1$ std. The row of a plot corresponds to the seed used to generate the data whilst the column corresponds to the dimension in this 5 dimensional tracking example.}
    
    \label{fig:fullELBOcomparison}
\end{figure}

\subsection{Sequential Variational Auto-Encoder Experiment} \label{sec:SeqVAEApdx}
A video demonstration of the sequential VAE model is available at \url{https://github.com/andrew-cr/online_var_fil}. 
\subsubsection{Experimental Details}
We parameterize $\text{NN}_\theta^f$ as a residual MLP which consists of 4 stacked layers of the form $f(x) = x + s\text{MLP}(x)$ where $\text{MLP}(x)$ is an MLP with a single hidden layer of hidden dimension 32 and $s$ is a learned scaling parameter. We parameterize $\text{NN}^g$ as a convolutional neural network, using the architecture suggested by \footnote{\url{https://github.com/deepmind/sonnet/blob/v2/examples/vqvae\_example.ipynb}} and implemented in PyTorch in \footnote{\url{https://github.com/karpathy/deep-vector-quantization} MIT License}. Specifically, the architecture consists of the following layers: 
\begin{itemize}
    \item Convolutional layer, 128 outputs channels, kernel size of 3 and padding of 1
    \item ReLU activation
    \item Two residual blocks consisting of
    \begin{itemize}
        \item Convolutional layer, 32 outputs channels, kernel size of 3, padding of 1
        \item ReLU activation
        \item Convolutional layer, 128 output channels, kernel size of 1, padding of 0
        \item Residual connection
    \end{itemize}
    \item Transposed convolution, 64 output channels, kernel size of 4, stride of 2 and padding of 1
    \item ReLU activation
    \item Transposed convolution, 3 output channels, kernel size of 4, stride of 2 and padding of 1
\end{itemize}
To pre-train the decoder we use the encoder architecture from the same sources and train using the standard minibatch ELBO objective to convergence.\\

We use KRR to represent $\hat{T}_t$. We use a KRR regularization parameter of $0.1$ and use an RBF kernel with learned bandwidth on a validation dataset. We update $\hat{T}_t$ using a dataset of size 512. For $\hat{S}_t$ we use a two hidden layer MLP with ReLU activations with hidden layer dimensions, 256 and 1024. The $\theta$ dimension is 8452. The MLP has an output dimension of 8453, the first 8452 outputs give the gradient direction and the last gives the log magnitude of the gradient. The MLP is trained on the regression dataset using a combination of a direction and magnitude loss. The direction loss is the negative cosine similarity whilst the magnitude loss is an MSE loss in log magnitude space. The two losses are then combined with equal weighting. This separation of the gradient into a magnitude and direction is similar to the gradient pre-processing described in \cite{learningtolearnbygradientdescentbygradientdescent}. The dataset size used for regression is 1024. We take minibatches of size 32 randomly sampled from this dataset and take 128 gradient steps on the $\hat{S}_t$ weights for each time step with a learning rate of 0.001.\\

We use mean field $q_t^{\phi_t}(x_t)$ and $q_t^{\phi_t}(x_{t-1}|x_t)$. The mean vector for $q_t^{\phi_t}(x_{t-1}|x_t)$ is given by an MLP with input $x_t$, and two hidden layers of dimension 64. The log standard deviation for $q_t^{\phi_t}(x_{t-1}|x_t)$ is learned directly and does not depend on $x_t$.\\

We set $U = V = 0.1\mathbb{I}$. We use a learning rate for $\theta$ updates of 0.001. We run 200 iterations of inner $\phi_t$ optimization at each time step, with a learning rate of $0.03$ for $q_t^{\phi_t}(x_t)$ statistics and $0.003$ for $q_t^{\phi_t}(x_{t-1}|x_t)$ weights. These experiments were run on a single RTX 3090 GPU.
\subsubsection{Average log likelihood}

\begin{figure}[t]
    \centering
    \includegraphics[width=0.5\textwidth]{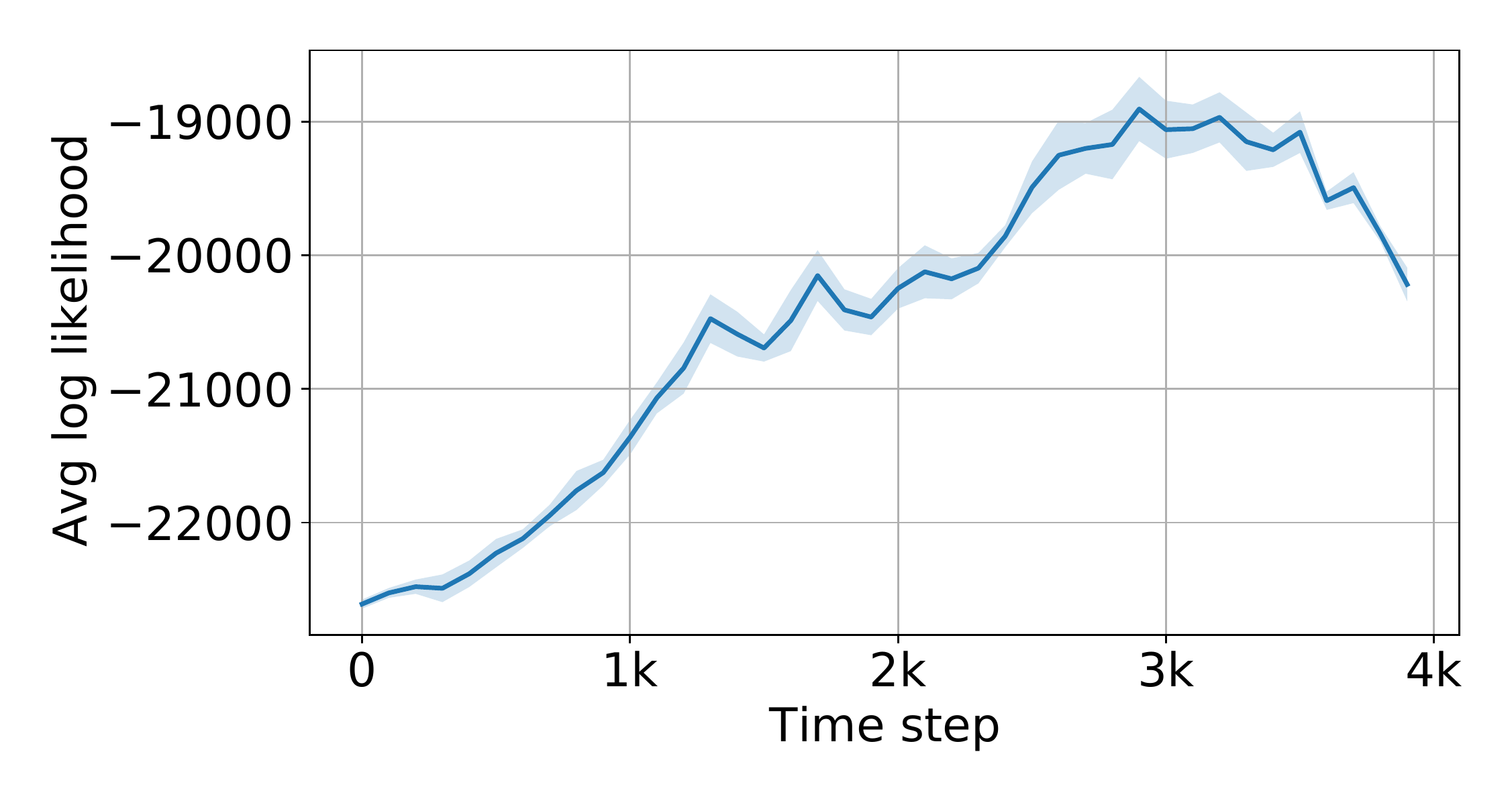}
    \caption{$\hat{\ell}_t^q(\theta_k) / t$ for $k=[1, \dots, t]$. Solid line is mean over three different seeds using the same dataset, transparent area is $\pm$ one standard deviation.}
    \label{fig:seqVAE_avg_log_likelihoods}
\end{figure}

To quantify the quality of the learned transition function $\text{NN}_\theta^f$ we can compute the approximate average log likelihood $\ell_t(\theta)/t$ using values of $\theta$ from different stages of training. The environment is highly non-stationary on the timescale considered because the dataset consists of video frames from an agent exploring a maze giving transitions that are diverse and not repeated. We therefore do not expect $\ell_t(\theta)/t$ to converge to a constant but its general trajectory is useful to quantify the agreement between the model and the observations. $\ell_t(\theta)/t$ is defined as
\begin{equation}
    \frac{\ell_t(\theta)}{t} = \frac{1}{t} \sum_{k=1}^t \log p_\theta(y_k|y^{k-1}) = \frac{1}{t} \sum_{k=1}^t \log \int p_\theta(x_{k-1}|y^{k-1}) f_\theta(x_k | x_{k-1}) g(y_k | x_k) \textrm{d}x_{k} \textrm{d}x_{k-1}.
\end{equation}
We approximate this quantity using the learned filtering distributions
\begin{equation}
    \frac{\ell_t^q(\theta)}{t} = \frac{1}{t} \sum_{k=1}^t \log \int q_{k-1}^{\phi_{k-1}}(x_{k-1}) f_\theta(x_k|x_{k-1}) g(y_k|x_k) \textrm{d}x_k \textrm{d}x_{k-1}
\end{equation}
which we can estimate through Monte Carlo
\begin{equation}
    \frac{\hat{\ell}_t^q(\theta)}{t} = \frac{1}{t} \sum_{k=1}^t \log \left\{ \sum_{i=1}^N g(y_k | x_k^n) \right\} \quad x_{k-1}^n, x_k^n \sim q_{k-1}^{\phi_{k-1}}(x_{k-1}) f_\theta(x_k | x_{k-1}).
\end{equation}
We note that $\hat{\ell}_t^q / t$ depends on $\phi_{1:t}$ so can only be computed at the end of the run. To monitor progress of $\theta$, we compute $\hat{\ell}_t^q(\theta_k) / t$ for $k = [1, \dots, t]$ where $\theta_k$ is the value of the model parameters during training at time step $k$. We plot the results in Figure \ref{fig:seqVAE_avg_log_likelihoods}. We find that indeed $\hat{\ell}_t^q(\theta_k) / t$ generally increases through training showing that our method can learn high dimensional model parameters. We conjecture that the few decreases in $\hat{\ell}_t^q(\theta_k) / t$ are due to online nature of the algorithm, if the current temporally local state of the system involves transitions that are not represented throughout the training sequence then updates in this region of time will cause an overall decrease in log likelihood when aggregated over the entire sequence. However, in the limit as $t\rightarrow \infty$ a recursive maximum likelihood approach, upon which our method is based, is expected to reach a local maximum of $\ell_t(\theta)/t$ under regularity conditions because local transients are averaged out in the long term.

\newpage

\section{Proofs} \label{sec:Proofs}
\subsection{Proof of \Cref{prop:ELBOrecursion}}
Recall from (\ref{eq:recursionq_t}) that we have
\begin{align}
    q_{t+1}^{\phi}(x_{1:t+1})=q_{t}^{\phi}(x_{1:t})m_{t+1}^{\phi}(x_{t+1}|x_t),
\end{align}
so
\begin{align}
    \log\frac{p_{\theta}(x_{1:t+1},y^{t+1})}{q_{t+1}^{\phi}(x_{1:t+1})}&=\log\frac{p_{\theta}(x_{1:t},y^{t})}{q_{t}^{\phi}(x_{1:t})}+\log \frac{f_\theta(x_{t+1}|x_t)g_\theta(y_{t+1}|x_{t+1})}{m_{t+1}^{\phi}(x_{t+1}|x_t)}\\
    &=\log\frac{p_{\theta}(x_{1:t},y^{t})}{q_{t}^{\phi}(x_{1:t})}+r_{t+1}^{\theta,\phi}(x_{t},x_{t+1}).
\end{align}
From the definition
\begin{equation}
     V^{\theta,\phi}_{t+1}(x_{t+1})=\mathbb{E}_{q_{t+1}^{\phi}(x_{1:t}|x_{t+1})}\left[ \log\frac{p_{\theta}(x_{1:t+1},y^{t+1})}{q_{t+1}^{\phi}(x_{1:t+1})}\right],
\end{equation}
we have directly 
\begin{equation}
    \mathcal{L}_{t+1}(\theta,\phi)=\mathbb{E}_{q_{t+1}^{\phi}(x_{t+1})}\left[V^{\theta,\phi}_{t+1}(x_{t+1})\right]
\end{equation}
and, crucially, 
\begin{align}
    V^{\theta,\phi}_{t+1}(x_{t+1})&:=\mathbb{E}_{q_{t+1}^{\phi}(x_{1:t}|x_{t+1})}\left[\log\frac{p_{\theta}(x_{1:t},y^{t})}{q_{t}^{\phi}(x_{1:t})}+r_{t+1}^{\theta,\phi}(x_{t},x_{t+1})\right]\\
    &=\mathbb{E}_{q_{t+1}^{\phi}(x_{t}|x_{t+1})}\left[\mathbb{E}_{q_{t}^{\phi}(x_{1:t-1}|x_{t})}\left[\log\frac{p_{\theta}(x_{1:t},y^{t})}{q_{t}^{\phi}(x_{1:t})}\right]+ r_{t+1}^{\theta,\phi}(x_{t},x_{t+1})\right]\\
    &=\mathbb{E}_{q_{t+1}^{\phi}(x_{t}|x_{t+1})}\left[ V^{\theta,\phi}_{t}(x_{t})+ r_{t+1}^{\theta,\phi}(x_{t},x_{t+1})\right].
\end{align}

\subsection{Proof of \Cref{prop:ELBOgradtheta}}
We have by direct differentiation that
\begin{align}
    \nabla_{\theta}\mathcal{L}_{t}(\theta,\phi)&=\mathbb{E}_{q_{t}^{\phi}(x_t)}[\nabla_\theta V^{\theta,\phi}_t(x_t)]\\
    &=\mathbb{E}_{q_{t}^{\phi}(x_{1:t})}\left[\nabla_{\theta}\log p_{\theta}(x_{1:t},y^t)\right]\\
    &=\mathbb{E}_{q^{\phi}_t(x_t)}\left[S^{\theta,\phi}_t(x_t) \right]
\end{align}
where 
\begin{equation}
    S^{\theta,\phi}_t(x_t):=\mathbb{E}_{q^{\phi}_{t}(x_{1:t-1}|x_{t})}\left[\nabla_{\theta}\log p_{\theta}(x_{1:t},y^t) \right].
\end{equation}
This quantity satisfies the recursion
\begin{align}
     S^{\theta,\phi}_{t+1}(x_{t+1})&=\mathbb{E}_{q^{\phi}_{t+1}(x_{1:t}|x_{t+1})}\left[\nabla_{\theta}\log p_{\theta}(x_{1:t+1},y^{t+1}) \right]\\
     &=\mathbb{E}_{q^{\phi}_{t+1}(x_{1:t}|x_{t+1})}\left[\nabla_{\theta}\log p_{\theta}(x_{1:t},y^{t})+ s^{\theta}_{t+1}(x_t,x_{t+1}) \right]\\
     &=\mathbb{E}_{q^{\phi}_{t+1}(x_{t}|x_{t+1})}\left[\mathbb{E}_{q^{\phi}_{t}(x_{1:t-1}|x_{t})}\left[\nabla_{\theta}\log p_{\theta}(x_{1:t},y^{t})  \right]+s^{\theta}_{t+1}(x_t,x_{t+1})\right]\\
     &=\mathbb{E}_{q^{\phi}_{t+1}(x_{t}|x_{t+1})}\left[S^{\theta,\phi}_t(x_t)+s^{\theta}_{t+1}(x_t,x_{t+1})\right]
\end{align}
for 
\begin{equation}
    s^{\theta}_{t+1}(x_{t},x_{t+1}):=\nabla_\theta \log f_{\theta}(x_{t+1}|x_{t})g_\theta(y_{t+1}|x_{t+1}).
\end{equation}
\subsection{Proof of \Cref{prop:ELBOgradphi}}
We have by a direct application of the reparameterization trick that 
\begin{equation}
\nabla_{\phi_t} \mathcal{L}_{t}(\theta,\phi_{1:t}) =\mathbb{E}_{\lambda(\epsilon_t)}[\nabla_{\phi_t} V^{\theta,\phi_{1:t}}_{t}(x_t(\phi_t; \epsilon_t))],
\end{equation}
where $V^{\theta,\phi_{1:t}}_{t}(x_t):=\mathbb{E}_{q_{t}^{\phi_{1:t}}(x_{1:t-1}|x_{t})}\left[ \log\frac{p_{\theta}(x_{1:t},y^{t})}{q_{t}^{\phi_{1:t}}(x_{1:t})}\right]$. We have  
\begin{align}
    V^{\theta,\phi_{1:t+1}}_{t+1}(x_{t+1}) & = \mathbb{E}_{q_{t+1}^{\phi_{t+1}}(x_t|x_{t+1})} \left[ V^{\theta,\phi_{1:t}}_{t}(x_t)+r_{t+1}^{\theta,\phi_{t:t+1}}(x_{t},x_{t+1}) \right] \\
    & = \mathbb{E}_{\lambda(\epsilon_t)} \left[ V^{\theta,\phi_{1:t}}_{t}(x_{t}(\phi_{t+1}; \epsilon_{t}, x_{t+1}))+r_{t+1}^{\theta,\phi_{t:t+1}}(x_{t}(\phi_{t+1}; \epsilon_{t}, x_{t+1}),x_{t+1}) \right]
\end{align}

Hence, 
\begin{multline}
\nabla_{\phi_{t+1}} V^{\theta,\phi_{1:t+1}}_{t+1}(x_{t+1}(\phi_{t+1}; \epsilon_{t+1})) \\
= \mathbb{E}_{\lambda(\epsilon_t)} \left[\frac{\partial}{\partial x_t} V^{\theta,\phi_{1:t}}_{t}(x_t)\Bigr|_{x_t=x_{t}(\phi_{t+1}; \epsilon_{t}, x_{t+1}(\phi_{t+1}; \epsilon_{t+1}))} \frac{\textup{d} x_{t}(\phi_{t+1}; \epsilon_{t}, x_{t+1}(\phi_{t+1}; \epsilon_{t+1}))}{\textup{d} \phi_{t+1}} \right. \\
 + \left. \nabla_{\phi_{t+1}} r_{t+1}^{\theta,\phi_{t:t+1}}(x_{t}(\phi_{t+1}; \epsilon_{t}, x_{t+1}(\phi_{t+1}; \epsilon_{t+1})),x_{t+1}(\phi_{t+1}; \epsilon_{t+1})) \right], 
\end{multline}
where 
\begin{equation}
    \frac{\partial}{\partial x_t} V^{\theta,\phi_{1:t}}_{t}(x_t)\Bigr|_{x_t=x_{t}(\phi_{t+1}; \epsilon_{t}, x_{t+1}(\phi_{t+1}; \epsilon_{t+1}))} = T^{\theta,\phi_{1:t}}_t(x_{t}(\phi_{t+1}; \epsilon_{t}, x_{t+1}(\phi_{t+1}; \epsilon_{t+1}))).
\end{equation}

For the forward recursion, 
\begin{align}
 & \quad T^{\theta,\phi_{1:t+1}}_{t+1}(x_{t+1}) = \frac{\partial}{\partial x_{t+1}} V^{\theta,\phi_{1:t+1}}_{t+1}(x_{t+1}) \\
 & = \mathbb{E}_{\lambda(\epsilon_t)}\left[\frac{\partial}{\partial x_{t+1}} V^{\theta,\phi_{1:t}}_{t}(x_{t}(\phi_{t+1}; \epsilon_{t}, x_{t+1}))+\nabla_{x_{t+1}} r_{t+1}^{\theta,\phi_{t:t+1}}(x_{t}(\phi_{t+1}; \epsilon_{t}, x_{t+1}),x_{t+1})\right] \\
 & = \begin{multlined}[t] \mathbb{E}_{\lambda(\epsilon_t)}\left[\frac{\partial}{\partial x_{t}} V^{\theta,\phi_{1:t}}_{t}(x_{t})\Bigr|_{x_t=x_{t}(\phi_{t+1}; \epsilon_{t}, x_{t+1})} \frac{\partial x_{t}(\phi_{t+1}; \epsilon_{t}, x_{t+1})}{\partial x_{t+1}} \right. \hspace{4cm}
 \\ + \left. \nabla_{x_{t+1}} r_{t+1}^{\theta,\phi_{t:t+1}}(x_{t}(\phi_{t+1}; \epsilon_{t}, x_{t+1}),x_{t+1})\right], \end{multlined}
\end{align}
where again 
\begin{equation}
    \frac{\partial}{\partial x_{t}}V^{\theta,\phi_{1:t}}_{t}(x_{t})\Bigr|_{x_t=x_{t}(\phi_{t+1}; \epsilon_{t}, x_{t+1})} = T^{\theta,\phi_{1:t}}_t(x_{t}(\phi_{t+1}; \epsilon_{t}, x_{t+1})).
\end{equation}

Here, $\frac{\textup{d} x_{t}(\phi_{t+1}; \epsilon_{t}, x_{t+1}(\phi_{t+1}; \epsilon_{t+1}))}{\textup{d} \phi_{t+1}}$, $\nabla_{\phi_{t+1}} r_{t+1}^{\theta,\phi_{t:t+1}}(x_{t}(\phi_{t+1}; \epsilon_{t}, x_{t+1}(\phi_{t+1}; \epsilon_{t+1})),x_{t+1}(\phi_{t+1}; \epsilon_{t+1}))$, $\nabla_{x_{t+1}} r_{t+1}^{\theta,\phi_{t:t+1}}(x_{t}(\phi_{t+1}; \epsilon_{t}, x_{t+1}),x_{t+1})$ denote total derivatives w.r.t. $\phi_{t+1}$ and $x_{t+1}$, which can all be computed using the reparameterization trick.

\newpage

\section{Amortization} \label{sec:AppendixAmortization}
In the main paper, we have focused on the case where we use a new set of variational parameters $\phi_t$ at each time step. This is conceptually simple and easy to use; however, it requires a new inner optimization run for each time step. In this section, we describe methods to amortize the $\phi$ optimization over time. This entails learning an amortization network which takes in $y$-observations and gives variational distribution statistics. With this network in hand, information from previous time steps regarding the relation between the $y$-observations and the variational statistics can be re-used. This results in computational savings over a method which treats each time step in isolation.

\subsection{Architecture} \label{sec:ApdxAmortizaitonArchitecture}
The variational distributions of interest $q_t^{\phi}(x_t | y^t)$ and $q_t^\phi(x_{t-1}|y^{t-1},x_t)$ approximate $p_\theta(x_t | y^t)$ and $p_\theta(x_{t-1}|y^{t-1},x_t)$ respectively. Therefore, to produce the $q_t^{\phi}$ statistics, we use a Recurrent Neural Network (RNN) to first encode the sequence of observations, $y^t$, creating a representation $h_t$. Then a network takes $h_t$ to give $q_t^{\phi}(x_t | y^t)$ statistics and a separate network takes $h_{t-1}$ and $x_t$ to give $q_t^{\phi}(x_{t-1}|y^{t-1},x_t)$ statistics.

In the simplest version of this architecture, we can directly take $h_t$ to be the statistics of the filtering distribution $q_t^{\phi}(x_t | y^t)$. In this case, two amortization networks are learned: one forward filtering network learns to output the filtering statistics $h_t$ given $h_{t-1}$ and $y_t$; another backward smoothing network learns to approximate the backward kernel $p_\theta(x_{t-1}|y^{t-1},x_t)$ given $h_{t-1}$ and $x_t$. Note that when we fix and detach $h_{t-1}$ at each time $t$, in a way this works like supervised learning, i.e. to learn to produce the same filtering and smoothing parameters as the non-amortized case, so the objective of the non-amortized case can be directly translated here to learn amortized networks. 

\subsection{Amortized Joint ELBO Maximization Approach} \label{sec:ApdxAmortizationUfunction}

One issue with using the non-amortized objective directly is that the amortized backward smoothing kernels $q^{\phi}_t(x_{t-1}|y^{t-1},x_t)$ are not learned to jointly maximize the joint ELBO $\mathcal{L}_{T}(\theta,\phi)$ over time. Here, we give an alternative, more rigorous treatment of the amortization objective. We assume in this section that we take $h_t$ as the filtering statistics and detach $h_t$ at each time step. 

For the amortized forward filtering network, we would like to maximize the objective $-\sum_{t=1}^{T}\text{KL}\left(q_{t}^{\phi}(x_{t}|y^{t})\parallel p_{\theta}(x_{t}|y^{t})\right)$. 
Since the KL is intractable, we learn $q_{t}^{\phi}(x_{t}|y^{t})$
by maximizing $\mathcal{L}_{t}(\theta,\phi)$ at each time step $t$. This is the same as the non-amortized case, which only utilizes the functional approximator $T^{\theta,\phi}_t(x_{t})$. 

For the amortized backward smoothing network, we would like to maximize $\mathcal{L}_{T}(\theta,\phi)$ jointly over time. To do this, we can similarly optimize $\mathcal{L}_{t}(\theta,\phi) - \mathcal{L}_{t-1}(\theta,\phi)$ at each time $t$ similar to learning $\theta$. For the gradients $\nabla_{\phi} \mathcal{L}_{t}(\theta,\phi)$, we have the following result: 

\begin{proposition}\label{prop:ELBOgradphiamortized}
The ELBO gradient $\nabla_{\phi} \mathcal{L}_{t}(\theta,\phi)$ satisfies for $t\geq 1$
\begin{align}
\nabla_{\phi} \mathcal{L}_{t}(\theta,\phi)=\nabla_{\phi} \mathbb{E}_{q_{t}^{\phi}(x_t|y^t)}[V^{\theta,\phi}_{t}(x_t)]=\mathbb{E}_{\lambda(\epsilon_t)}[\nabla_{\phi} V^{\theta,\phi}_{t}(x_t(\phi; \epsilon_t))].
\end{align}

Additionally, one has
\begin{multline}\label{eq:recursionQphiamortized}
\nabla_{\phi} V^{\theta,\phi}_{t+1}(x_{t+1}(\phi; \epsilon_{t+1}))  = \mathbb{E}_{\lambda(\epsilon_t)} \left[T^{\theta,\phi}_t(x_{t}(\phi; \epsilon_{t}, x_{t+1}(\phi; \epsilon_{t+1}))) \frac{\textup{d} x_{t}(\phi; \epsilon_{t}, x_{t+1}(\phi; \epsilon_{t+1}))}{\textup{d} \phi} \right. \\ 
 + \left. U^{\theta,\phi}_t(x_{t}(\phi; \epsilon_{t}, x_{t+1}(\phi; \epsilon_{t+1}))) + \nabla_{\phi} r_{t+1}^{\theta,\phi}(x_{t}(\phi; \epsilon_{t}, x_{t+1}(\phi; \epsilon_{t+1})),x_{t+1}(\phi; \epsilon_{t+1}))\right], 
\end{multline}

where
$T^{\theta,\phi}_t(x_{t})  := \frac{\partial}{\partial x_t} V^{\theta,\phi}_{t}(x_t)$, $U^{\theta,\phi}_t(x_{t})  := \nabla_\phi V^{\theta,\phi}_{t}(x_t)$
satisfy the forward recursions  
\begin{multline}\label{eq:recursionTamortized}
T^{\theta,\phi}_{t+1}(x_{t+1}) =  \mathbb{E}_{\lambda(\epsilon_t)}\left[T^{\theta,\phi}_t(x_{t}(\phi; \epsilon_{t}, x_{t+1})) \frac{\partial x_{t}(\phi; \epsilon_{t}, x_{t+1})}{\partial x_{t+1}}\right. \\ 
+ \left.\nabla_{x_{t+1}} r_{t+1}^{\theta,\phi}(x_{t}(\phi; \epsilon_{t}, x_{t+1}),x_{t+1})\right], 
\end{multline}
\begin{multline}\label{eq:recursionUamortized}
U^{\theta,\phi}_{t+1}(x_{t+1}) =  \mathbb{E}_{\lambda(\epsilon_t)}\left[T^{\theta,\phi}_t(x_{t}(\phi; \epsilon_{t}, x_{t+1})) \frac{\partial x_{t}(\phi; \epsilon_{t}, x_{t+1})}{\partial \phi}\right. \\ 
+ \left. U^{\theta,\phi}_t(x_{t}(\phi; \epsilon_{t}, x_{t+1})) + \nabla_{\phi} r_{t+1}^{\theta,\phi}(x_{t}(\phi; \epsilon_{t}, x_{t+1}),x_{t+1})\right].  
\end{multline}
\end{proposition}

\begin{proof}

By a direct application of the reparameterization trick,  
\begin{equation}
\nabla_{\phi} \mathcal{L}_{t}(\theta,\phi)=\mathbb{E}_{\lambda(\epsilon_t)}[\nabla_{\phi} V^{\theta,\phi}_{t}(x_t(\phi; \epsilon_t))].
\end{equation}
By Proposition \ref{prop:ELBOrecursion}, 
\begin{equation}
    V^{\theta,\phi}_{t+1}(x_{t+1})=\,\mathbb{E}_{q_{t+1}^{\phi}(x_{t}|y^t, x_{t+1})}[V^{\theta,\phi}_{t}(x_{t})+r_{t+1}^{\theta,\phi}(x_{t},x_{t+1})].
\end{equation}
Hence, using the chain rule, 
\begin{multline}
\nabla_{\phi} V^{\theta,\phi}_{t+1}(x_{t+1}(\phi; \epsilon_{t+1}))  = \mathbb{E}_{\lambda(\epsilon_t)} \left[ \frac{\partial}{\partial x_t} V^{\theta,\phi}_{t}(x_t)\Bigr|_{x_t=x_{t}(\phi; \epsilon_{t}, x_{t+1}(\phi; \epsilon_{t+1}))} \frac{\textup{d} x_{t}(\phi; \epsilon_{t}, x_{t+1}(\phi; \epsilon_{t+1}))}{\textup{d} \phi} \right.\\
+ \left.\frac{\partial}{\partial \phi} V^{\theta,\phi}_{t}(x_t)\Bigr|_{x_t=x_{t}(\phi; \epsilon_{t}, x_{t+1}(\phi; \epsilon_{t+1}))} + \nabla_{\phi} r_{t+1}^{\theta,\phi}(x_{t}(\phi; \epsilon_{t}, x_{t+1}(\phi; \epsilon_{t+1})),x_{t+1}(\phi; \epsilon_{t+1}))\right]. 
\end{multline}
The functions $\frac{\partial}{\partial x_t} V^{\theta,\phi}_{t}(x_t), \frac{\partial}{\partial \phi} V^{\theta,\phi}_{t}(x_t)$ are defined as $T^{\theta,\phi}_t(x_t), U^{\theta,\phi}_t(x_t)$ respectively. 

$T^{\theta,\phi}_t(x_t)$ follows the same forward recursion as in \Cref{prop:ELBOgradphi}. For  $U^{\theta,\phi}_t(x_t)$, we have
\begin{align}
 U^{\theta,\phi}_{t+1}(x_{t+1}) & = \frac{\partial}{\partial \phi} V^{\theta,\phi}_{t+1}(x_{t+1}) \\
 &= \frac{\partial}{\partial \phi} \mathbb{E}_{q_{t+1}^{\phi}(x_{t}|y^t, x_{t+1})}[V^{\theta,\phi}_{t}(x_{t})+r_{t+1}^{\theta,\phi}(x_{t},x_{t+1})] \\
 &= \mathbb{E}_{\lambda(\epsilon_t)}\left[\nabla_\phi \left(V^{\theta,\phi}_{t}(x_{t}(\phi; \epsilon_{t}, x_{t+1}))+r_{t+1}^{\theta,\phi}(x_{t}(\phi; \epsilon_{t}, x_{t+1}),x_{t+1})\right)\right] \\
 & = \begin{multlined}[t] \mathbb{E}_{\lambda(\epsilon_t)}\left[\frac{\partial}{\partial x_t} V^{\theta,\phi}_{t}(x_t)\Bigr|_{x_t=x_{t}(\phi; \epsilon_{t}, x_{t+1})} \frac{\partial x_{t}(\phi; \epsilon_{t}, x_{t+1})}{\partial \phi} \right. \hspace{3.1cm}  \\ 
+ \left. \frac{\partial}{\partial \phi} V^{\theta,\phi}_{t}(x_t)\Bigr|_{x_t=x_{t}(\phi; \epsilon_{t}, x_{t+1})}  + \nabla_{\phi} r_{t+1}^{\theta,\phi}(x_{t}(\phi; \epsilon_{t}, x_{t+1}),x_{t+1})\right] \end{multlined} \\
 & = \begin{multlined}[t] \mathbb{E}_{\lambda(\epsilon_t)}\left[T^{\theta,\phi}_t(x_{t}(\phi; \epsilon_{t}, x_{t+1})) \frac{\partial x_{t}(\phi; \epsilon_{t}, x_{t+1})}{\partial \phi}  \right.  \hspace{4.4cm} \\ 
+ \left. U^{\theta,\phi}_t(x_{t}(\phi; \epsilon_{t}, x_{t+1}))+ \nabla_{\phi} r_{t+1}^{\theta,\phi}(x_{t}(\phi; \epsilon_{t}, x_{t+1}),x_{t+1})\right] . \end{multlined}
\end{align}

\end{proof}

Note that in the simplest case where we take $h_t$ as the filtering statistics and detach $h_t$ at each time step, these are reflected in the corresponding gradient computations and function updates. In this case, $\nabla_\phi \mathcal{L}_{t}(\theta,\phi) - \nabla_\phi  \mathcal{L}_{t-1}(\theta,\phi)$ reduces to 
\begin{multline}
\nabla_{\phi}\mathbb{E}_{q_{t}^{\phi}(x_{t}|y^t)q_{t}^{\phi}(x_{t-1}|y^{t-1}, x_{t})}\left[r_{t}^{\theta,\phi}(x_{t-1},x_{t})\right]\\
+ \mathbb{E}_{\lambda(\epsilon_{t-1})\lambda(\epsilon_{t})}\left[T_{t-1}^{\theta,\phi}\left(x_{t-1}(\phi;\epsilon_{t-1},x_{t}(\phi;\epsilon_{t}))\right)\frac{\textup{d}x_{t-1}(\phi;\epsilon_{t-1},x_{t}(\phi;\epsilon_{t}))}{\textup{d}\phi}\right]\\
+ \mathbb{E}_{q_{t}^{\phi}(x_{t}|y^t)q_{t}^{\phi}(x_{t-1}|y^{t-1}, x_{t})}\left[U_{t-1}^{\theta,\phi}\left(x_{t-1}\right)\right] -\mathbb{E}_{q_{t-1}^{\phi}(x_{t-1}|y^{t-1})}\left[U_{t-1}^{\theta,\phi}\left(x_{t-1}\right)\right]
\end{multline}
since $q_{t-1}^\phi(x_{t-1}|y^{t-1})$ has been detached. Note that the terms on the last line represent the difference between two expectations of the $U^{\theta,\phi}_{t-1}(x_{t-1})$ function. We refer to this approach as ``Ours (TU)''.  Without the final line, the objective reduces to the same objective as the non-amortized case, and we refer to this approach as ``Ours (T)''. We demonstrate the applicability of these methods on the Chaotic RNN task. For both the forward filtering and backward smoothing networks, we use a MLP with 2 hidden layers and 100 neurons in each layer. As shown in Table \ref{tab:crnn_rmse_amortized}, the amortized networks are able to achieve similar filtering accuracy as the non-amortized case. Both of the proposed methods achieve lower errors during training time and test time. 

\begin{table}
\centering
\caption{Root Mean Squared Error between filtering mean and true state on the CRNN task using amortized models. For test time errors, we rerun the trained models from the start of data without further optimization. The resulting errors are comparable to those of the non-amortized models. }
\label{tab:crnn_rmse_amortized}
\vspace{1em}

\begin{tabular}{cccccc}
\toprule 
$d_{x}$ &  & AELBO-1 & AELBO-2 & Ours (T) & Ours (TU)\tabularnewline
\midrule
\midrule 
\multirow{2}{*}{5} & Filter RMSE (train time) & 0.1158\textpm 0.0011 & 0.1039\textpm 0.0006 & \textbf{0.1032\textpm 0.0004} & \textbf{0.1031\textpm 0.0002}\tabularnewline
\cmidrule{2-6} \cmidrule{3-6} \cmidrule{4-6} \cmidrule{5-6} \cmidrule{6-6} 
 & Filter RMSE (test time) & 0.1170\textpm 0.0022 & 0.1064\textpm 0.0012 & \textbf{0.1048\textpm 0.0005} & \textbf{0.1056\textpm 0.0013}\tabularnewline
\bottomrule
\end{tabular}
\end{table}

\subsection{Semi-Amortized Approach} \label{sec:ApdxSemiAmortized}
An alternative approach to amortization is to return to the exact same gradient computations as in the main paper. However, instead of $\phi_t$ corresponding directly to the statistics of $q_t^{\phi_t}$, it corresponds to the parameters of the RNN and MLPs which produce $q_t^{\phi_t}$ statistics from the observations $y^t$. For each time step, we run the inner optimization and `overfit' the RNN/MLPs to the current set of observations. Overfitting in this context means that the networks are optimized to produce as accurate as possible $q_t^{\phi_t}$ statistics for this time step (just as in the main paper), but since only the current time step is considered in the optimization, they are not forced to generalize to other time steps and $y$-observations.\\

This may seem contradictory with the aims of amortization, namely using computational work spent during previous time steps to reduce the inference load at the current time step. However, if the RNN/MLPs are initialized at their optimized values from the previous time step, the previous optimization cycles can be thought of as a type of pre-training. This makes the inner optimization problem progressively easier in terms of computation required for a certain level of accuracy.\\

We demonstrate this idea using the linear Gaussian application where the distance to the true filtering distributions can be calculated analytically. We use an RNN and MLPs to generate $q_t^{\phi_t}$ statistics as described above and optimize each $\phi_t$ for 100 steps at each time step. Figure \ref{fig:amortized_zeroshot} plots the absolute error in the mean of $q_t^{\phi_t}(x_t)$ (averaged over dimension) versus time step for 5 different points within each inner optimization routine. Looking at the zero shot performance (at the start of each time step, before any optimization) we see that over time, the amortization networks are able to produce more and more accurate statistics without any updates using the current observations. This shows that this naive approach to amortization can indeed provide useful cost savings in the long-run.\\

Using this approach, we were able to reproduce the model learning results shown in Figure \ref{fig:linear_model_learning} but with fewer iterations per time step. This was also achieved for the Sequential VAE example using a convolutional RNN to encode video frames and MLPs to generate variational statistics. A visually plausible transition function could be successfully learned in this semi-amortized fashion. 

\begin{figure}
    \centering
    \includegraphics[width=8cm]{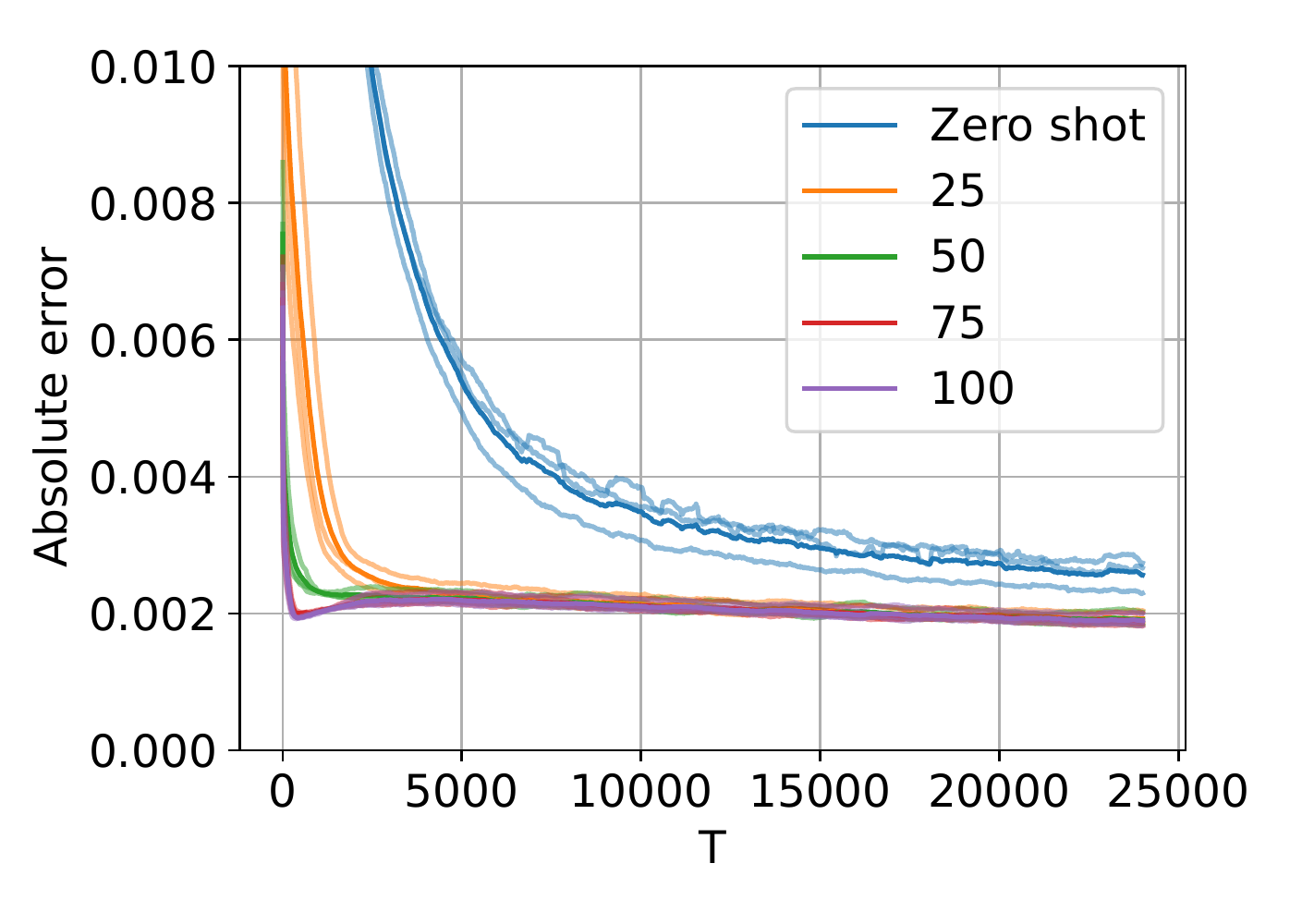}
    \caption{Absolute error of the mean of $q_t^{\phi_t}(x_t)$ averaged over dimension versus time step. 5 different points in each inner optimization process are plotted, zero shot is before any optimization steps, 25 is after 25 optimization steps, and so on. Different seeds are shown as translucent with the mean over seed shown in full color. The plateau in absolute error is due to the inherent limitations of stochastic gradient descent with a fixed learning rate. To contextualize the absolute error value, states and state transitions are on the order of $\sim 0.1$. The lines are smoothed using a uniform kernel of width 1000.}
    \label{fig:amortized_zeroshot}
\end{figure}

\subsection{Gradient Computations} \label{sec:ApdxAmortizationGradients}
All approaches to amortization require computing gradients of the form $\frac{\textup{d} x_t}{\textup{d} \phi}$ with $x_t$ being sampled using the reparameterization trick with statistics from $q_t^{\phi}(x_t)$. When we use an RNN to calculate these statistics, calculating this derivative requires backpropagating through all previous observations $y^t$. This results in a linearly increasing computational cost in time. To avoid this, we detach the RNN state $h_{t-H}$ from the computational graph at some fixed window into the past, $H$. When we roll out the RNN to calculate statistics from time $t$, we simply initialize at $h_{t-H}$ and treat it as a constant. When the algorithm proceeds to the next time step, $h_{t+1-H}$ is then kept constant at its most recent value during the previous time step. More sophisticated methods for online training of RNNs are also possible, we refer to \cite{marschall2020unified} for a survey.

\section{Broader Impact} \label{sec:BroaderImpact}
We propose a generic methodology for performing online variational inference and parameter estimation. Historically, filtering has been used in a huge variety of applications, some with large societal impacts. The same filtering algorithms can both help predict future weather patterns but could also be used in weapons guidance systems. Our current methodology remains in the research stage but as further developments are made that make it more practically applicable, it is important to fully consider these possible societal effects.
\end{document}